\documentclass{article}
\usepackage[round,comma]{natbib}
\usepackage{times}
\usepackage{amsmath,amssymb}
\usepackage{bbm}      
\usepackage{booktabs}
\usepackage{cases}
\usepackage[round,comma]{natbib}
\usepackage{multirow}
\usepackage{graphicx}

\newcommand{\BlackBox}{\rule{1.5ex}{1.5ex}}  
\newenvironment{proof}{\par\noindent{\bf Proof\ }}{\hfill\BlackBox\\[2mm]}
\newtheorem{example}{Example} 
\newtheorem{theorem}{Theorem}
\newtheorem{lemma}[theorem]{Lemma}

\newtheorem{remark}[theorem]{Remark}

\newtheorem{definition}[theorem]{Definition}

\ifnum\pdfoutput>0
  \DeclareGraphicsExtensions{.pdf}
\else
  \DeclareGraphicsExtensions{.eps}
\fi

\newlength{\minipagewidth}
\newcommand{\bookbox}[1]{\small
\par\medskip\noindent
\framebox[\columnwidth]{
\begin{minipage}{\minipagewidth} {#1} \end{minipage} } \par\medskip }

\renewcommand{\phi}{\varphi}
\renewcommand{\epsilon}{\varepsilon}
\newcommand{\eps}{\epsilon}

\newcommand{\K}{\mathrm{KL}}
\newcommand{\KL}{\mathrm{KL}}
\newcommand{\Conv}{\text{Conv}}

\newcommand{\oR}{\overline{R}}
\newcommand{\E}{\mathbb{E}}
\newcommand{\EXP}{\mathbb{E}}
\newcommand{\IND}{\mathbbm{1}}

\newcommand{\N}{\mathbb{N}}
\newcommand{\R}{\mathbb{R}}

\newcommand{\cA}{\mathcal{A}}

\newcommand{\cD}{\mathcal{D}}

\newcommand{\cL}{\mathcal{L}}

\newcommand{\cS}{\mathcal{S}}

\newcommand{\cB}{\mathcal{B}}
\renewcommand{\P}{\mathbb{P}}

\newcommand{\argmax}{\mathop{\mathrm{argmax}}}
\newcommand{\argmin}{\mathop{\mathrm{argmin}}}

\newcommand{\ul}{\underline}

\newcommand{\tl}{\tilde{\ell}}
\newcommand{\Tl}{\tilde{T}(\ell_t)}

\newcommand\lam{\lambda}
\newcommand\set[2]{\{#1,\dots,#2\}}
\newcommand{\Tr}{\mathrm{Tr}}
\newcommand\dsV{\mathbb{V}} 
\newcommand\Var{{\dsV\text{ar}}\,}

\newcommand\ba{\mathbf{a}}
\newcommand\ra{\rightarrow}
\newcommand\fracl[2]{{(#1)}/{#2}}

\newcommand\cleb{{\sc cleb}}

\allowdisplaybreaks

\begin{document}

\title{Minimax Policies for Combinatorial Prediction Games}

\author{Jean-Yves Audibert \\
Imagine, Univ. Paris Est, and Sierra, \\
CNRS/ENS/INRIA, Paris, France \\
\texttt{\small audibert@imagine.enpc.fr}
\\ \\
S{\'e}bastien Bubeck\\
Centre de Recerca Matem{\`a}tica \\
Barcelona, Spain \\
\texttt{\small sbubeck@crm.cat}
\\ \\
G{\'a}bor Lugosi\\
ICREA and Pompeu Fabra University \\
Barcelona, Spain \\
\texttt{\small lugosi@upf.es}}

\maketitle

\begin{abstract}
We address the online linear optimization problem when the
actions of the forecaster are represented by binary vectors.
Our goal is to understand the magnitude of the minimax regret
for the worst possible set of actions. We study the problem
under three different assumptions for the feedback: 
full information, and the partial information models of the
so-called ``semi-bandit'', and ``bandit'' problems. 
We consider both $L_\infty$-, and $L_2$-type of restrictions for
the losses assigned by the adversary.

We formulate a general strategy using Bregman projections 
on top of a potential-based gradient descent, which  
generalizes the ones studied in the series of 
papers \cite{GLLO07, DHK08, AHR08, CL09, HW09, KWK10, UNK10, KRS10} and \cite{AB10}. 
We provide simple
proofs that recover most of the previous results. 
We propose new upper bounds for the semi-bandit game. 
Moreover we derive lower bounds for all three feedback assumptions. 
With the only exception of the bandit game, the upper and lower bounds
are tight, up to a constant factor.
Finally, we answer a question asked by 
\cite{KWK10} by showing that the exponentially 
weighted average forecaster is suboptimal against $L_{\infty}$ adversaries.
\end{abstract}

\section{Introduction}

In the sequential decision making problems considered in this paper,
at each time instance $t=1,\ldots,n$, the forecaster chooses,
possibly in a randomized way, an action from a given set $\cS$
where $\cS$ is a subset of the $d$-dimensional hypercube $\{0,1\}^d$.
The action chosen by the forecaster at time $t$ is denoted by
 $V_t=(V_{1,t},\dots,V_{d,t})\in \cS$.
 Simultaneously to the forecaster, the adversary chooses a loss
vector $\ell_t=(\ell_{1,t},\dots,\ell_{d,t}) \in[0,+\infty)^d$ and the
loss incurred by the forecaster is $\ell_t^TV_t$.
The goal of the forecaster is to minimize the expected cumulative
loss $\EXP \sum_{t=1}^n \ell_t^TV_t$ where the expectation is taken with
respect to the forecaster's internal randomization. 
This problem is an instance of an ``online linear optimization'' 
problem\footnote{In online linear optimization problems, the action set
is often not restricted to be a subset of $\{0,1\}^d$ but can be 
an arbitrary subset of $\R^d$. However, in the most interesting cases,
actions are naturally represented by Boolean vectors and we
restrict our attention to this case.}, see, e.g., 
\cite{AK04, MB04, KV05, GLLO07, DHK08, AHR08, CL09, HW09, KWK10, UNK10} and \cite{KRS10}

We consider three variants of the problem, distinguished by 
the type of information that becomes available to the forecaster 
at each time instance, after taking an action.
(1) In the {\em full information game} the forecaster observes
the entire loss vector  $\ell_t$;
(2) in the {\em semi-bandit game} only those components $\ell_{i,t}$ of
  $\ell_t$ are observable for which $V_{i,t}=1$;
(3) in the {\em bandit game} only the total loss
  $\ell_t^TV_t$ becomes available to the forecaster.

We refer to these problems as {\em combinatorial prediction games}.
All three prediction games are sketched in
Figure \ref{fig:comband}. For all three games, we define the regret\footnote{For the full
  information game, one can directly upper bound the stronger notion
  of regret $\E \sum_{t=1}^n \ell_t^T V_t - \E \min_{v \in \cS}
  \sum_{t=1}^n \ell_t^T v $ which is always larger than
  $\oR_n$. However, for partial information games, this requires more
  work.} of the forecaster as
$$\oR_n = \E \sum_{t=1}^n \ell_t^T V_t - \min_{v \in \cS} \E \sum_{t=1}^n \ell_t^T v .$$ 
 
\begin{figure}[t]
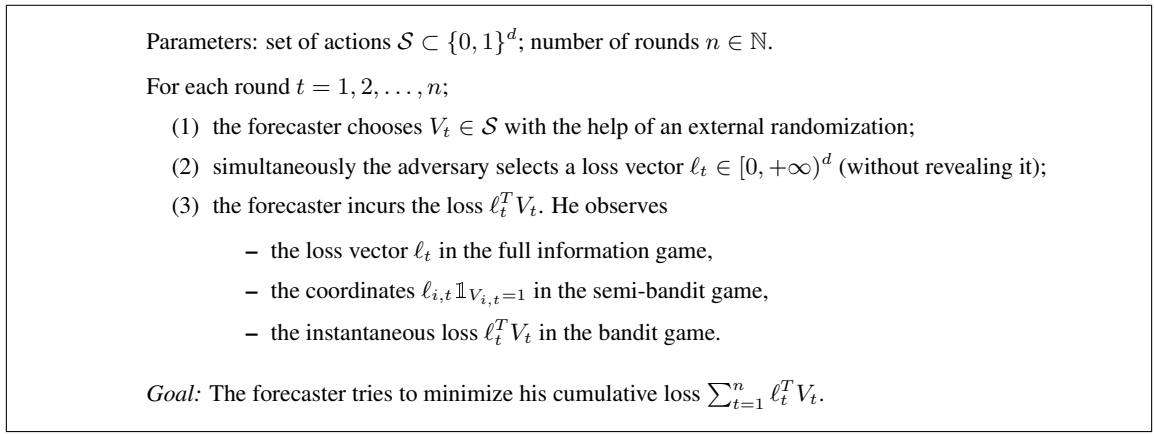

\bookbox{\small
{Parameters:} set of actions $\cS \subset \{0,1\}^d$; number of rounds $n \in \N$.

\medskip\noindent
For each round $t=1,2,\ldots,n$;
\begin{itemize}
\item[(1)]
the forecaster chooses $V_t \in \cS$ with the help of an external randomization;
\item[(2)]
\mbox{simultaneously the adversary selects a loss vector $\ell_t \in [0,+\infty)^d$ (without revealing it)};
\item[(3)]
the forecaster incurs the loss $\ell_t^T V_t $. He observes
\begin{itemize}
\item the loss vector $\ell_t$ in the full information game,
\item the coordinates $\ell_{i,t} \IND_{V_{i,t} = 1}$ in the semi-bandit game,
\item the instantaneous loss $\ell_t^T V_t $ in the bandit game.
\end{itemize}
\end{itemize}

\medskip\noindent
{\em Goal:} The forecaster tries to minimize his cumulative loss $\sum_{t=1}^n \ell_t^T V_t$.
}
\caption{Combinatorial prediction games.}
\label{fig:comband}
\end{figure}

In order to make meaningful statements about the regret, one needs
to restrict the possible loss vectors the adversary may assign. 
We work with 
two different natural assumptions that have been considered in the 
literature:

\noindent
{\bf $L_{\infty}$ assumption:} here we assume that $\|\ell_t\|_{\infty} \leq 1$ for all $t=1,\ldots,n$

\noindent
{\bf $L_2$ assumption:} assume that $\ell_t^T v \leq 1$ for all $t=1,\ldots,n$ and $v\in\cS$.

Note that, without loss of generality, we may assume that for all 
$i\in\{1,\ldots,d\}$,
there exists $v\in\cS$ with $v_i=1$, and then 
the $L_2$ assumption implies the $L_{\infty}$ assumption.
 
The goal of this paper is to study the {\em minimax regret}, 
that is, the performance of the forecaster that minimizes the regret
for the worst possible sequence of loss assignments.
This, of course, depends on the set $\cS$ of actions. 
Our aim is to determine the order of magnitude of the
minimax regret for the most difficult set to learn.  
More precisely, for a given game, if we write
sup for the supremum over all allowed adversaries (that is, either
$L_{\infty}$ or $L_2$ adversaries) and inf for the infimum over all
forecaster strategies for this game, we are interested in the 
maximal minimax regret
$$R_n = \max_{\cS \subset \{0,1\}^d} \inf \sup \oR_n~.$$

Note that in this paper we do not restrict our attention to
computationally efficient algorithms.  The following example
illustrates the different games 
that we introduced above.

\begin{example} 
Consider the well studied example of {\em path planning} in which, at
every time instance, the forecaster chooses a path from one fixed
vertex to another in a graph.  At each time, a loss is assigned to
every edge of the graph and, depending on the model of the feedback,
the forecaster observes either the losses of all edges, the losses of
each edge on the chosen path, or only the total loss of the chosen
path. The goal is to minimize the total loss for any sequence of
loss assignments. This problem can be cast as a combinatorial prediction game
in dimension $d$ for $d$ the number of edges in the graph.
\end{example}

\begin{table}[t]
\begin{tabular}{c|c|c|c|c|c|c}
   & \multicolumn{3}{|c|}{$L_{\infty}$} & \multicolumn{3}{|c}{$L_2$} \\
  \hline
  & { Full Info} & { Semi-Bandit} & {Bandit} & { Full Info} & { Semi-Bandit} & {Bandit} \\
  \hline
  {Lower Bound} & $d \sqrt{n}$ & $d \sqrt{n}$ & $\mathbf{d^{3/2} \sqrt{n}}$ & $\sqrt{d n}$ & $\sqrt{d n}$ & $d \sqrt{n}$\\
  \hline
  {Upper Bound} & $d \sqrt{n}$ & $\mathbf{d \sqrt{n}}$ & $d^{5/2} \sqrt{n}$ & $\sqrt{d n}$ & $\mathbf{\sqrt{d n \log d}}$ & $d^{3/2} \sqrt{n}$
  \end{tabular}
\caption{Bounds on $R_n$ proved in this paper (up to constant factor). 
The new results are set in bold.}
\label{table:1}
\end{table}

\begin{table}[t]
\begin{tabular}{c|c|c|c|c|c|c}
   & \multicolumn{3}{|c|}{$L_{\infty}$} & \multicolumn{3}{|c}{$L_2$} \\
  \hline
  & { Full Info} & { Semi-Bandit} & {Bandit} & { Full Info} & { Semi-Bandit} & {Bandit} \\
  \hline
  {{\sc exp2}} & $d^{3/2} \sqrt{n}$ & $d^{3/2} \sqrt{n}$ & $d^{5/2} \sqrt{n}$ & $\sqrt{d n}$ & $\mathbf{d \sqrt{n}}$ * & $d^{3/2} \sqrt{n}$\\
  \hline
    {{\sc linexp}} & $d \sqrt{n}$ & $\mathbf{d \sqrt{n}}$ & $\mathbf{d^2 n^{2/3}}$ & $\mathbf{\sqrt{d n}}$ & $\mathbf{d \sqrt{n}}$ * & $\mathbf{d^2 n^{2/3}}$ \\
 \hline
{{\sc linpoly}} & $\mathbf{d \sqrt{n}}$ & $\mathbf{d \sqrt{n}}$ & - & $\mathbf{\sqrt{d n}}$ & $\mathbf{\sqrt{d n \log d}}$ & - 
  \end{tabular}
\caption{Upper bounds on $\oR_n$ for specific forecasters. The new results are in bold. 
We also show that the bound for {\sc exp2} in the full information game is unimprovable. Note that the bound for (Bandit, {\sc linexp}) is very weak. 
The bounds with * become $\sqrt{dn\log d}$ if we restrict our attention to sets $\cS$ that are ``almost symmetric'' in the sense that for some $k$,
$\cS\subset \big\{ v\in\{0,1\}^d : \sum_{i=1}^d v_i \le k \big\}$ and 
$\Conv(\cS)\cap \big[\frac{k}{2d};1\big]^d\neq\emptyset$.
}
\label{table:2}
\end{table}

Our contribution is threefold. First, we propose a variant of the
algorithm used to track the 
best linear predictor \citep{HeWa98} that
is well-suited to our combinatorial prediction games.  This leads to
an algorithm called \cleb\ that generalizes various approaches that have been
proposed.  This new point of view on algorithms that were defined for
specific games (only the full information game, or only the standard
multi-armed bandit game) allows us to generalize them easily to all
combinatorial prediction games, leading to new algorithms such as
{\sc linpoly}. This algorithmic contribution leads to our second main
result, the improvement of the known upper bounds for
the semi-bandit game. This point of view also leads to a different
proof of the minimax $\sqrt{n d}$ regret bound in the standard
$d$-armed bandit game that is much simpler than the one provided in
\cite{AB10}. A summary of the bounds proved in this paper can be found
in Table \ref{table:1} and Table \ref{table:2}. In addition we prove 
several lower bounds. First, we establish lower
bounds on the minimax regret in all three games and under both types of
adversaries, whereas only the cases ($L_2 / L_{\infty}$, Full
Information) and ($L_2$, Bandit) were previously treated in the
literature. Moreover we also answer a question of \cite{KWK10} by
showing that the traditional exponentially weighted average forecaster
is suboptimal against $L_{\infty}$ adversaries.

In particular, 
this paper leads to the following (perhaps unexpected) conclusions:
\begin{itemize}
\item 
The full information game is as hard as the semi-bandit game. More
precisely, in terms of $R_n$, the price that one pays for the limited
feedback of the semi-bandit game compared to the full information game
is only a constant factor (or a $\sqrt{\log d}$ factor for the $L_2$
setting).
\item 
In the full information and semi-bandit game, the traditional
exponentially weighted average forecaster is provably suboptimal for
$L_{\infty}$ adversaries while it is optimal for $L_2$ adversaries in
the full information game.
\item 
Denote by $\cA_2$ (respectively $\cA_\infty$) the set of
adversaries that satisfy the $L_2$ assumption (respectively the
$L_{\infty}$ assumption). We clearly have $\cA_2 \subset \cA_{\infty}
\subset d \cA_2$. We prove that, in the full information game, $R_n$
gains an additional factor of $\sqrt{d}$ at each inclusion. In the
semi-bandit game, we show that the same statement remains true up to a
logarithmic factor.
\end{itemize}

\paragraph{Notation.} 
The convex hull of $\cS$ is denoted $\Conv(\cS)$.

\section{Combinatorial learning with Bregman projections} 

In this section we introduce a general forecaster that we
call \cleb\ (Combinatorial LEarning with Bregman
projections). Every
forecaster investigated in this paper is a special case of \cleb.

Let $\cD$ be a convex subset of $\R^d$ with nonempty interior $\text{Int}(\cD)$
and boundary $\partial \cD$.
\begin{definition}
We call Legendre any function $F:\cD\ra\R$ such that 
\begin{itemize}
\item[(i)] $F$ is strictly convex and admits continuous first partial
derivatives on $\text{Int}(\cD)$
\item[(ii)]  For any $u\in \partial \cD$, for any $v\in \text{Int}(\cD)$, we have
$$\lim_{s\ra 0,s>0} (u-v)^T\nabla F\big((1-s)u+sv\big) = +\infty.$$
\end{itemize}
\end{definition}
The Bregman divergence $D_F: \cD\times \text{Int}(\cD)$ associated to a Legendre function $F$ is defined by
  $$
  D_F(u,v) = F(u) - F(v) - (u-v)^T\nabla F(v).
  $$ 
We consider the algorithm \cleb\ described in
Figure \ref{fig:cleb}.  The basic idea is to use a
potential-based gradient descent \eqref{eq:wp} followed by a
projection \eqref{eq:proj} with respect to the Bregman divergence of
the potential onto the convex hull of $\cS$ to ensure that the
resulting weight vector $w_{t+1}$ can be viewed as 
$w_{t+1} = \E_{V \sim p_{t+1}} V$ for some distribution $p_{t+1}$ on $\cS$. The
combination of Bregman projections with potential-based gradient
descent was first used in \cite{HeWa98}.  Online learning with Bregman
divergences without the projection step has a long history (see
Section 11.11 of \cite{CesLug06}).  As discussed below, \cleb\ may be
viewed as a generalization of the forecasters {\sc linexp} and {\sc inf}.

\begin{figure}[!h]
\bookbox{
%

Parameters: 
\begin{itemize}
\item a Legendre function $F$ defined on $\cD$ with 
$\Conv(\cS)\cap \text{Int}(\cD) \neq \emptyset$
\item $w_1 \in \Conv(\cS)\cap \text{Int}(\cD)$
\end{itemize}

\medskip\noindent
For each round $t=1,2,\ldots,n$;
\begin{itemize}
\item[(a)]
Let $p_t$ be a distribution on the set $\cS$ such that $w_t = \E_{V \sim p_t} V$.
\item[(b)] 

Draw a random action $V_t$ according to the distribution $p_t$ and observe 
\begin{itemize}
\item the loss vector $\ell_t$ in the full information game,
\item the coordinates $\ell_{i,t} \IND_{V_{i,t} = 1}$ in the semi-bandit game,
\item the instantaneous loss $\ell_t^T V_t $ in the bandit game.
\end{itemize}
\item[(c)]
Estimate the loss $\ell_t$ by $\tl_t$.
For instance, one may take
\begin{itemize}
\item $\tilde{\ell}_t=\ell_t$ in the full information game,
\item $\tilde{\ell}_{i,t} = \frac{\ell_{i,t}}{\sum_{v\in\cS:v_i=1} p_{t}(v)} V_{i,t}$ in the semi-bandit game,
\item $\tilde{\ell_t} = P_t^+ V_t V_t^T \ell_t,$ 
with $P_t = \E_{v \sim p_t} (v v^T)$
in the bandit game.
\end{itemize}
\item[(d)]
Let $w'_{t+1}\in\text{Int}(\cD)$ satisfying
  \begin{align} \label{eq:wp}
  \nabla F(w'_{t+1}) = \nabla F(w_t) - \tilde{\ell}_t.
  \end{align} 
\item[(e)]
Project the weight vector $w'_{t+1}$ defined by \eqref{eq:wp} on the convex hull of $\cS$:
\begin{equation} \label{eq:proj}
w_{t+1} \in \argmin_{w \in \Conv(\cS)\cap\text{Int}(\cD)} D_F(w,w_{t+1}') .
\end{equation}
\end{itemize}
}
\caption{Combinatorial learning with Bregman projections (\cleb).}\label{fig:cleb}
\end{figure}

The Legendre conjugate $F^*$ of $F$ is defined by
  $F^*(u)=\sup_{v\in \cD} \big\{ u^Tv-F(v)\big\}$.
The following theorem establishes the first step of all
upper bounds for the regret of \cleb.

\begin{theorem} \label{th:cleb}
\cleb\ satisfies for any $u \in \Conv(\cS)\cap\cD$,
\begin{align} 
\sum_{t=1}^n \tilde{\ell}_t^T w_t & - \sum_{t=1}^n \tilde{\ell}_t^T u 
\le D_F(u,w_1)+\sum_{t=1}^n D_{F^*}(\nabla F(w_t)-\tl_t,\nabla F(w_t)) \label{eq:clebb}.
\end{align}
\end{theorem}

\begin{proof}
By applying the definition of the Bregman divergences (or equivalently using Lemma 11.1 of \cite{CesLug06}), we obtain
  \begin{align*}
  \tl_t^Tw_t-\tl_t^Tu & = (u-w_t)^T\big(\nabla F(w'_{t+1}) - \nabla F(w_t) \big)\\
  & = D_F(u,w_t)+D_F(w_t,w'_{t+1})-D_F(u,w'_{t+1}).
  \end{align*}
By the Pythagorean theorem (Lemma 11.3 of \cite{CesLug06}), we have 
  $D_F(u,w_{t+1}') \ge D_F(u,w_{t+1}) + D_F(w_{t+1},w_{t+1}'),$ hence
  \begin{align*}
  \tl_t^Tw_t-\tl_t^Tu \le D_F(u,w_t)+D_F(w_t,w'_{t+1})-D_F(u,w_{t+1})-D_F(w_{t+1},w_{t+1}').
  \end{align*}
Summing over $t$ then gives  
  \begin{multline} \label{eq:gen}
  \sum_{t=1}^n \tl_t^Tw_t-\sum_{t=1}^n \tl_t^Tu \le D_F(u,w_1)-D_F(u,w_{n+1})
    +\sum_{t=1}^n \big(D_F(w_t,w'_{t+1}) -D_F(w_{t+1},w_{t+1}')\big).\quad
  \end{multline}
By the nonnegativity of the Bregman divergences, we get 
\begin{align*} 
\sum_{t=1}^n \tilde{\ell}_t^T w_t & - \sum_{t=1}^n \tilde{\ell}_t^T u 
\le D_F(u,w_1)+\sum_{t=1}^n D_F(w_t,w'_{t+1}).
\end{align*}
From Proposition 11.1 of \cite{CesLug06}, we have
  $
  D_F(w_t,w'_{t+1})=D_{F^*}\big(\nabla F(w_t)-\tl_t,\nabla F(w_t)\big),
  $
which concludes the proof.
\end{proof}

As we will see below,
by the equality $\E \sum_{t=1}^n \tl_t^TV_t=\E \sum_{t=1}^n \tl_t^T
w_t$, and provided that $\tl_t^T V_t$ and $\tl_t^T u$ are unbiased
estimates of $\E \ell_t^T V_t$ and $\E \ell_t^T u$, Theorem
\ref{th:cleb} leads to an upper bound on the regret $\oR_n$ of \cleb,
which allows us to obtain the bounds of Table \ref{table:2} by using
appropriate choices of $F$. Moreover, if $F$ admits an Hessian,
denoted $\nabla^2F$, that is always invertible, then one can prove that
up to a third-order term \big(in $\tl_t$\big), the regret bound can be
written as:
\begin{equation} \label{eq:intuition}
\sum_{t=1}^n \tilde{\ell}_t^T w_t - \sum_{t=1}^n \tilde{\ell}_t^T u \lessapprox D_F(u,w_1)+\sum_{t=1}^n \tl_t^T \left( \nabla^2 F(w_t) \right)^{-1} \tl_t .
\end{equation}

In this paper, we restrict our attention to the combinatorial learning
setting in which $\cS$ is a subset of $\{0,1\}^d$. However, one should
note that this specific form of $\cS$ plays no role in the definition
of \cleb, meaning that the algorithm on Figure \ref{fig:cleb} can be
used to handle general online linear optimization problems, where
$\cS$ is any subset of $\R^d$.

\section{Different instances of \cleb}

\begin{figure*}[t]
\begin{center}
\includegraphics[height=6cm]{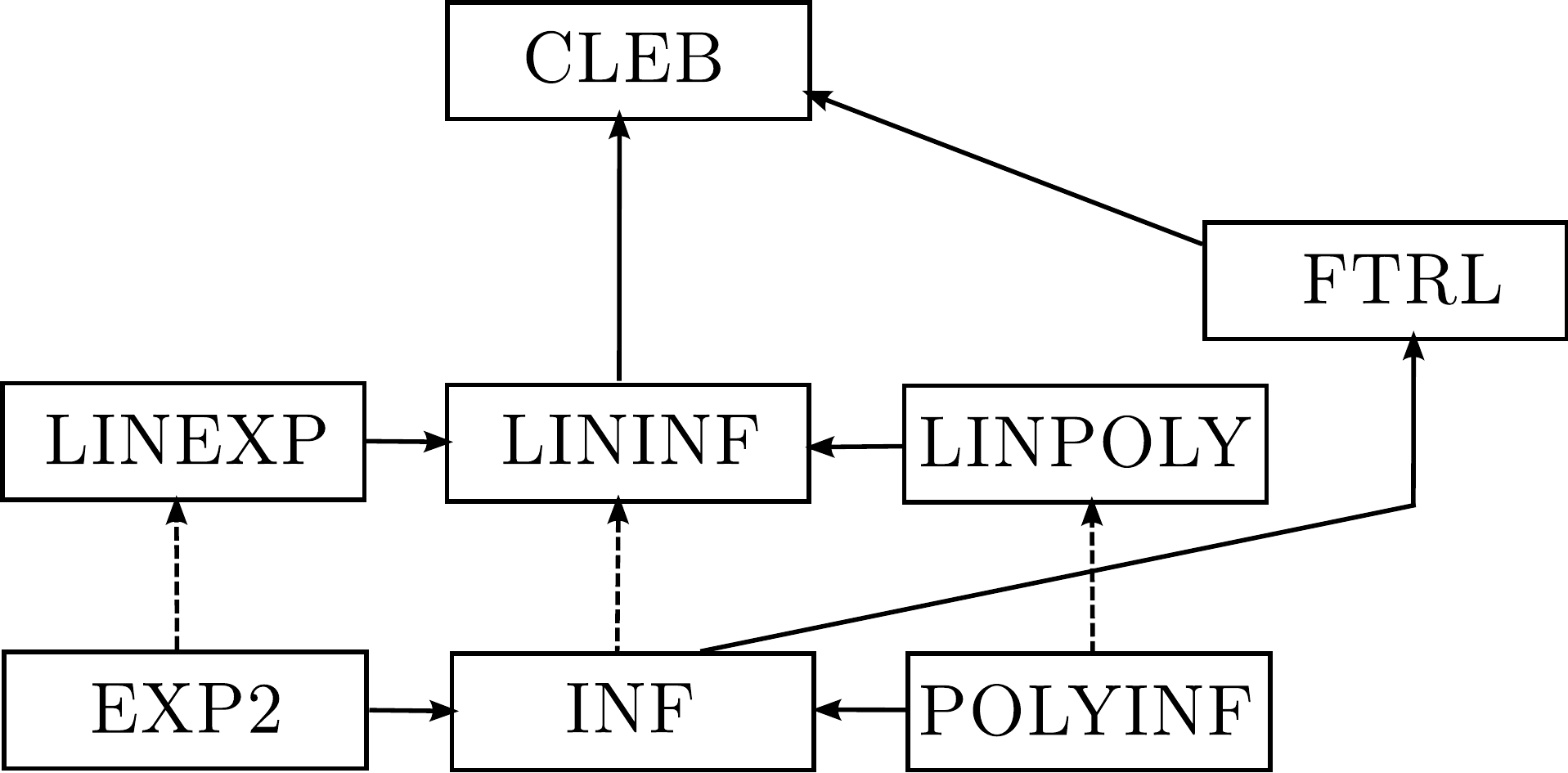}
\caption{The figure sketches the relationship of the algorithms
  studied in this paper with arrows representing ``is a special case
  of''.  Dotted arrows indicate that the link is obtained by
  ``expanding'' $\cS$, that, is seeing $\cS$ as the set of basis
  vector in $\R^{|\cS|}$ rather than seeing it as a (structured)
  subset of $\{0,1\}^{d}$ (see Section \ref{sec:expd}). The six
  algorithms on the bottom use a Legendre function with a diagonal
  Hessian. On the contrary, the {\sc ftrl} algorithm (see Section
  \ref{sec:ftrl}) may consider Legendre functions more adapted to the
  geometry of the convex hull of $\cS$. {\sc polyinf} is the algorithm considered
  in Theorem \ref{th:polyinf}.
\label{fig:organi}}
\end{center}
\end{figure*}

In this section we describe several instances of \cleb\, and relate
them to existing algorithms.  Figure \ref{fig:organi} summarizes the
relationship between the various algorithms introduced below.

\subsection{{\sc exp2} (Expanded Exponentially weighted average forecaster)} \label{sec:expd}
\begin{figure}[t!]
\bookbox{
{\em {\sc exp2}:}

\medskip\noindent

Parameter: Learning rate $\eta$.

\medskip\noindent
Let $w_1=\big(\frac1{|\cS|},\hdots,\frac1{|\cS|}\big) \in \R^{|\cS|}$.

\medskip\noindent
For each round $t=1,2,\ldots,n$;
\begin{itemize}
\item[(a)]
Let $p_t$ the distribution on $\cS$ such that $p_t(v)=w_{v,t}$ for any $v\in\cS$.
\item[(b)]
Play $V_t \sim p_t$ and observe 
\begin{itemize}
\item the loss vector $\ell_t$ in the full information game,
\item the coordinates $\ell_{i,t} \IND_{V_{i,t} = 1}$ in the semi-bandit game,
\item the instantaneous loss $\ell_t^T V_t $ in the bandit game.
\end{itemize}
\item[(c)]
Estimate the loss vector $\ell_t$ by $\tl_t$.
For instance, one may take
\begin{itemize}
\item $\tl_t=\ell_t$ in the full information game,
\item $\tilde{\ell}_{i,t} = \frac{\ell_{i,t}}{\sum_{v\in\cS:v_i=1} p_{v,t}} V_{i,t}$ in the semi-bandit game,
\item $\tilde{\ell_t} = P_t^+ V_t V_t^T \ell_t,$ 
with $P_t = \E_{v \sim p_t} (v v^T)$
in the bandit game.
\end{itemize}
\item[(d)]
Update the weights, for all $v \in \cS$, 
$$w_{v,t+1}=\frac{\exp(- \eta \tl_t^Tv) w_{v,t}}{\sum_{u\in\cS} \exp(- \eta \tl_t^Tu) w_{u,t}}.$$ 
\end{itemize}
}
\caption{{\sc exp2} forecaster.}\label{fig:Exp2}
\end{figure}

The simplest approach to combinatorial prediction games is to consider
each vertex of $\cS$ as an independent expert, and then apply a strategy
designed for the expert problem. We call {\sc exp2} the resulting
strategy when one uses the traditional exponentially weighted average
forecaster (also called Hedge, \cite{FS97}), see Figure
\ref{fig:Exp2}. In the full information game, {\sc exp2} corresponds
to Expanded Hedge defined in \cite{KWK10}, where it was studied under
the $L_{\infty}$ assumption. It was also studied in the full
information game under the $L_2$ assumption in \cite{DHK08}.  In the
semi-bandit game, {\sc exp2} was studied in \cite{GLLO07} under the
$L_{\infty}$ assumption. Finally in the bandit game, {\sc exp2}
corresponds to the strategy proposed by \cite{DHK08} and also
to the ComBand strategy, studied under the $L_{\infty}$
assumption in \cite{CL09} and under the $L_2$ assumption in
\cite{CL10}. (These last strategies differ in how the losses are estimated.)

{\sc exp2} is a \cleb\ strategy in dimension $|S|$ that uses $\cD =
[0,+\infty)^{|S|}$ and the function $F: u \mapsto \frac1\eta
  \sum_{i=1}^{|S|} u_i \log(u_i)$, for some $\eta>0$ (this can be
  proved by using the fact that the Kullback-Leibler projection on the
  simplex is equivalent to a $L_1$-normalization). The following theorem
  shows the regret bound that one can obtain for {\sc exp2} (for instance with Theorem \ref{th:lininf} applied to the case where 
$\cS$ is replaced by $\cS'=\big\{ u\in\{0,1\}^{|\cS|} : \sum_{v\in\cS} u_v=1\big\}$).

\begin{theorem} \label{th:Exp2}
For the {\sc exp2} forecaster, provided that
$\E \tl_t=\ell_t$, we have
	$$
	\oR_n \le \frac{\log(|\cS|)}{\eta} 
		+ \frac{\eta}2 \sum_{t=1}^n \sum_{v \in \cS} \E\big[p_{t}(v) (\tl_t^Tv)^2 
		\max\big(1,\exp(-\eta \tl_t^Tv)\big)\big].
	$$
\end{theorem}


\subsection{{\sc linexp} (Linear Exponentially weighted average forecaster)}
We call {\sc linexp} the \cleb\ strategy that uses $\cD =
[0,+\infty)^d$ and the function $F: u \mapsto \frac1\eta \sum_{i=1}^d
  u_i \log(u_i)$ associated to the Kullback-Leibler divergence, for
  some $\eta>0$. In the full information game, {\sc linexp}
  corresponds to Component Hedge defined in \cite{KWK10}, where it was
  studied under the $L_{\infty}$ assumption. In the semi-bandit game,
  {\sc linexp} was studied in \cite{UNK10, KRS10} under the
  $L_{\infty}$ assumption, and for the particular set $\cS$ with all
  vertices of $L_1$ norm equal to some value $k$.

\subsection{{\sc ftrl} (Follow the Regularized Leader)} \label{sec:ftrl}
If $\Conv(\cS) \subset \cD$ and $w_1 \in \argmin_{w\in\cD} F(w)$, steps (d) and (e) 
are equivalent to
$$
w_{t+1}\in\argmin_{w \in \Conv(\cS)} \left( \sum_{s=1}^t \tl_s^T w+F(w)\right),
$$
showing that in this case \cleb\ can be interpreted as a regularized
follow-the-leader algorithm. This type of algorithm was studied in \cite{AR09} in the
full information and bandit setting (see also the lecture notes \cite{RT08}).
A survey of {\sc ftrl} strategies for the
full information game can be found in \cite{Haz10}. In the bandit
game, {\sc ftrl} with $F$ being a self-concordant barrier function 
and a different estimate than the one proposed in Figure \ref{fig:cleb}
was studied in \cite{AHR08}.

\subsection{{\sc lininf} (Linear Implicitly Normalized Forecaster)}
Let $f:\R^d \ra \R$.
The function $f$ has a diagonal Hessian
if and only if it can be written as 
$f(u) = \sum_{i=1}^d f_i(u_i)$, for some twice differentiable functions $f_i:\R\ra\R$, $i=1,\dots,d$.
The Hessian is called exchangeable when the functions 
$f_1'',\dots,f_d''$ are identical.
In this case, up to adding an affine function of $u$ (note that 
this does not alter
neither the Bregman divergence nor \cleb), we have $f(u) =
\sum_{i=1}^d g(u_i)$ for some twice differentiable function $g$.  In
this section, we consider this type of Legendre functions.  
To underline the surprising
link\footnote{detailed in Appendix \ref{sec:inf}.} with the Implicitly
Normalized Forecaster proposed in \cite{AB10}, we consider $g$ of the
form $x\mapsto \int_{\cdot}^{x} \psi^{-1}(s)ds$, and will refer to the
algorithm presented hereafter as {\sc lininf}.

\begin{definition}
Let $\omega \ge 0$.
A function 
$\psi: (-\infty,a) \rightarrow \R^*_+$ for some 
$a\in\R\cup\{+\infty\}$ is called an 
$\omega$-potential if and only if it is convex,
 continuously differentiable, and satisfies
\begin{align*} 
& \lim_{x\ra -\infty} \psi(x)=\omega &&
\lim_{x\ra a} \psi(x)= +\infty \notag \\
& \psi' > 0 && \int_{\omega}^{\omega+1} |\psi^{-1}(s)|ds <+\infty.
\end{align*}
\end{definition}

\begin{theorem} \label{th:lininf}
Let $\omega \ge0$ and let $\psi$ be an $\omega$-potential function.
The function $F$ defined on $\cD= [\omega, +\infty)^d$ by 
$F(u)=\sum_{i=1}^d \int_{\omega}^{u_i} \psi^{-1}(s)ds$
is Legendre. The associated \cleb\ satisfies, for any 
$u \in \Conv(\cS)\cap\cD$,
\begin{multline} 
  \sum_{t=1}^n \tilde{\ell}_t^T w_t - \sum_{t=1}^n \tilde{\ell}_t^T u 
    \le D_F(u,w_1)
    +\frac12\sum_{t=1}^n \sum_{i=1}^d \tl_{i,t}^2 
    \max\Big(\psi'\big(\psi^{-1}(w_{i,t})\big) ,
    \psi'\big(\psi^{-1}(w_{i,t})-\tl_{i,t}\big) \Big), \quad
    \label{eq:psib}
  \end{multline}
where for any $(u,v)\in\cD\times\text{Int}(\cD)$,
  \begin{equation}\label{eq:dfpsi}
  D_F(u,v)=\sum_{i=1}^d \bigg(\int_{v_i}^{u_i} \psi^{-1}(s)ds-(u_i-v_i)\psi^{-1}(v_i) \bigg).
  \end{equation}
In particular, when the estimates $\tl_{i,t}$ are nonnegative, we have
\begin{align} 
  \sum_{t=1}^n \tilde{\ell}_t^T w_t - \sum_{t=1}^n \tilde{\ell}_t^T u 
    \le D_F(u,w_1)+\sum_{t=1}^n \sum_{i=1}^d \frac{\tl_{i,t}^2}{2(\psi^{-1})'(w_{i,t})}. 
    \label{eq:psic}
  \end{align}
\end{theorem}

\begin{proof}
It is easy to check that $F$ is a Legendre function and that 
\eqref{eq:dfpsi} holds. We also have
	$
  \nabla F^*(u)= (\nabla F)^{-1}(u)= \big(\psi(u_1),\dots,\psi(u_d)\big),
	$
hence
	$$
  D_{F^*}(u,v)=\sum_{i=1}^d \bigg(\int_{v_i}^{u_i} \psi(s)ds-(u_i-v_i)\psi(v_i)\bigg).
  $$
From the Taylor-Lagrange expansion, we have
	$
	D_{F^*}(u,v)\le\sum_{i=1}^d \max_{s\in[u_i,v_i]} \frac12 \psi'(s) (u_i-v_i)^2.
	$
Since the function $\psi$ is convex, we have
	$
	\max_{s\in[u_i,v_i]} \psi'(s) \le \psi'\big(\max(u_i,v_i)\big),
	$
which gives the desired results.
\end{proof}

Note that {\sc linexp} is an instance of {\sc lininf} 
with $\psi: x\mapsto\exp(\eta x)$. On the other hand,
\cite{AB10} recommend the choice $\psi(x)=(-\eta x)^{-q}$
with $\eta>0$ and $q>1$ since it leads to the minimax optimal
rate $\sqrt{n d}$ for the standard
$d$-armed bandit game \big(while the best bound for Exp3 
is of the order of $\sqrt{n d \log d}$\big).
This corresponds to a function $F$ of the form 
  $F(u)= - \frac{q}{(q-1)\eta} \sum_{i=1}^d u_i^{\fracl{q-1}q}$. We refer to the corresponding \cleb\ as {\sc linpoly}.
In Appendix \ref{sec:inf} we show that a simple application of Theorem \ref{th:lininf} proves that {\sc linpoly} with $q=2$
satisfies $\oR_n \le 2 \sqrt{2nd}$. This improves on the bound $\oR_n \le 8 \sqrt{nd}$ obtained in Theorem 11 of \cite{AB10}.

\section{Full Information Game} \label{sec:fullinfo}

This section details the upper bounds of the forecasters {\sc exp2}, {\sc linexp} and
{\sc linpoly} under the $L_2$ and $L_\infty$ assumptions for the full information game. 
All results are gathered in Table \ref{table:2} (page \pageref{table:2}). The proofs can be found in Appendix \ref{fi:app}. 
Up to numerical constants, the results concerning ({\sc exp2}, $L_2$ and $L_\infty$) and ({\sc linexp}, $L_\infty$) appeared or can be easily derived from respectively 
\cite{DHK08} and \cite{KWK10}.

\begin{theorem}[{\sc linexp}, $L_\infty$] \label{fi:linexp:infty}
Under the $L_\infty$ assumption, for {\sc linexp} with $\tl_t=\ell_t$, $\eta=\sqrt{2/n}$ and
$w_1 = \argmin_{w \in \Conv(\cS)} D_F\big(w,(1,\dots,1)^T\big),$ 
we have
	$$
	\oR_n \le d\sqrt{2n}.
	$$
\end{theorem}

\begin{theorem}[{\sc linexp}, $L_2$] \label{fi:linexp:2}
Under the $L_2$ assumption, for {\sc linexp} with $\tl_t=\ell_t$, $\eta=\sqrt{2d/n}$ and
$w_1 = \argmin_{w \in \Conv(\cS)} D_F\big(w,(1,\dots,1)^T\big),$ 
we have
	$$
	\oR_n \le \sqrt{2nd}.
	$$
\end{theorem}

\begin{theorem}[{\sc linpoly}, $L_\infty$] \label{fi:linpoly:infty}
Under the $L_\infty$ assumption, for {\sc linpoly} with $\tl_t=\ell_t$, 
$\eta=\sqrt{\frac2{q(q-1)n}}$ and
$w_1 = \argmin_{w \in \Conv(\cS)} D_F\big(w,(1,\dots,1)^T\big),$
we have
	$$
	\oR_n \le d\sqrt{\frac{2qn}{q-1}}.
	$$
\end{theorem}

\begin{theorem}[{\sc linpoly}, $L_2$] \label{fi:linpoly:2}
Under the $L_2$ assumption, for {\sc linpoly} with $\tl_t=\ell_t$, 
$\eta=\sqrt{\frac{2d}{q(q-1)n}}$ and
$w_1 = \argmin_{w \in \Conv(\cS)} D_F\big(w,(1,\dots,1)^T\big),$ 
we have
	$$
	\oR_n \le \sqrt{\frac{2qdn}{q-1}}.
	$$
\end{theorem}

\begin{theorem}[{\sc exp2}, $L_\infty$] \label{fi:exp2:infty}
Under the $L_\infty$ assumption, for {\sc exp2} with $\tl_t=\ell_t$, 
we have
	$$
	\oR_n \le \frac{d\log 2}{\eta} 
		+ \frac{\eta nd^2}2.
	$$
In particular for $\eta=\sqrt{\frac{2\log 2}{nd}}$, we have
	$
	\oR_n \le \sqrt{2d^3 n \log 2}.
	$

\end{theorem}

From Theorem \ref{lb:expinfty}, the above upper bound is tight, and consequently
there exists $\cS$ for which the algorithm {\sc exp2} is not minimax optimal in the full information game under the
$L_\infty$ assumption.

\begin{theorem}[{\sc exp2}, $L_2$] \label{fi:exp2:2}
Under the $L_2$ assumption, for {\sc exp2} with $\tl_t=\ell_t$, 
we have
	$$
	\oR_n \le \frac{d\log 2}{\eta} 
		+ \frac{\eta n}2.
	$$
In particular for $\eta=\sqrt{\frac{2d\log 2}{n}}$, we have
	$
	\oR_n \le \sqrt{2d n \log 2}.
	$

\end{theorem}

\section{Semi-Bandit Game} \label{sec:semibandit}

This section details the upper bounds of the forecasters {\sc exp2}, {\sc linexp} and
{\sc linpoly} under the $L_2$ and $L_\infty$ assumptions for the semi-bandit game. 
These bounds are gathered in Table \ref{table:2} (page \pageref{table:2}). The proofs can be found in Appendix \ref{sb:app}.
Up to the numerical constant, the result concerning ({\sc exp2}, $L_\infty$) appeared in 
\cite{GLLO07} in the context of the online shortest path problem. 
\cite{UNK10} and \cite{KRS10} studied the semi-bandit problem under the $L_\infty$ assumption
for action sets of the form $\cS = \big\{ v\in\{0,1\}^d : \sum_{i=1}^d v_i = k \big\}$ for some value $k$. Their common algorithm corresponds to {\sc linexp} and the bounds are of order $\sqrt{k n d \log(d/k)}$. Our upper bounds for the regret of {\sc linexp} extend these results to more general sets of arms and to the $L_2$ assumption. 

\begin{theorem}[{\sc linexp}, $L_\infty$] \label{sb:linexp:infty}
Under the $L_\infty$ assumption, for {\sc linexp} with 
$\tl_{i,t}=\ell_{i,t} \frac{V_{i,t}}{w_{i,t}}$, $\eta=\sqrt{2/n}$ and
$w_1 = \argmin_{w \in \Conv(\cS)} D_F\big(w,(1,\dots,1)^T\big),$
we have
	$$
	\oR_n \le d\sqrt{2n}.
	$$
\end{theorem}
Since the $L_2$ assumption implies the $L_\infty$ assumption, we also have $\oR_n \le d\sqrt{2n}$ under the $L_2$ assumption. 

Let us now detail how {\sc linexp} behaves for almost symmetric action sets as defined below.
\begin{definition}
The set $\cS \subset\{0,1\}^d$ is called almost symmetric if for some $k\in\set{1}{d}$,
$\cS\subset \big\{ v\in\{0,1\}^d : \sum_{i=1}^d v_i \le k \big\}$ and 
$\Conv(\cS)\cap \big[\frac{k}{2d};1\big]^d\neq\emptyset$.
The integer $k$ is called the order of the symmetry.
\end{definition}
The set $\cS = \big\{ v\in\{0,1\}^d : \sum_{i=1}^d v_i = k \big\}$
considered in \cite{UNK10} and \cite{KRS10} is 
a particular almost symmetric set.

\begin{theorem}[{\sc linexp}, almost symmetric $\cS$] \label{sb:linexp:infty2}
Let $\cS$ be an almost symmetric set of order $k\in\set{1}{d}$.
Consider {\sc linexp} with 
$\tl_{i,t}=\ell_{i,t} \frac{V_{i,t}}{w_{i,t}}$
and 
$w_1 = \underset{w \in \Conv(\cS)}{\argmin} D_F\big(w,(\frac{k}d,\dots,\frac{k}d)^T\big).$
Let $\cL=\max\big(\log\big(\frac{d}k\big),1\big)$.\vspace{-.1cm}
\begin{itemize}
\item
Under the $L_\infty$ assumption, taking $\eta= \sqrt{\frac{2k\cL}{nd}}$, we have
	$
	\oR_n \le \sqrt{2knd\cL}.
	$
\item
Under the $L_2$ assumption, taking $\eta=k\sqrt{\frac{\cL}{nd}} $, we have
	$
	\oR_n \le 2\sqrt{nd\cL}.
	$
\end{itemize}
\end{theorem}
In particular, it means that under the $L_2$ assumption, there is a gain in the regret bound of a factor $\sqrt{d/\cL}$ when the set of actions is an almost symmetric set of order $k$.

\begin{theorem}[{\sc linpoly}, $L_\infty$] \label{sb:linpoly:infty}
Under the $L_\infty$ assumption, for {\sc linpoly} with $\tl_{i,t}=\ell_{i,t} \frac{V_{i,t}}{w_{i,t}}$, 
$\eta=\sqrt{\frac2{q(q-1)n}}$ and
$w_1 = \argmin_{w \in \Conv(\cS)} D_F\big(w,(1,\dots,1)^T\big),$
we have
	$$
	\oR_n \le d\sqrt{\frac{2qn}{q-1}}.
	$$
\end{theorem}

\begin{theorem}[{\sc linpoly}, $L_2$] \label{sb:linpoly:2}
Under the $L_2$ assumption, for {\sc linpoly} with $\tl_{i,t}=\ell_{i,t} \frac{V_{i,t}}{w_{i,t}}$, 
$\eta=\sqrt{\frac{2d^{\frac1q}}{q(q-1)n}}$ and
$w_1 = \argmin_{w \in \Conv(\cS)} D_F\big(w,(1,\dots,1)^T\big),$
we have
	$$
	\oR_n \le \sqrt{\frac{2qnd}{q-1}d^{1-\frac1q}}.
	$$
In particular, for $q=1+(\log d)^{-1}$,
we have
	$
	\oR_n \le \sqrt{2nde\log(ed)}.
	$ 
\end{theorem}

\begin{theorem}[{\sc exp2}, $L_\infty$] \label{sb:exp2:infty}
Under the $L_\infty$ assumption, for the {\sc exp2} forecaster described in Figure \ref{fig:Exp2} using $\tl_{i,t}=\ell_{i,t}\frac{V_{i,t}}{w_{i,t}}$, 
we have
	$$
	\oR_n \le \frac{d\log 2}{\eta} 
		+ \frac{\eta nd^2}2.
	$$
In particular for $\eta=\sqrt{\frac{2\log 2}{nd}}$, we have
	$
	\oR_n \le \sqrt{2d^3 n \log 2}.
	$
\end{theorem}

The corresponding lower bound is given in Theorem \ref{lb:expinfty}.

\begin{theorem}[{\sc exp2}, $L_2$] \label{sb:exp2:2}
Under the $L_2$ assumption, for {\sc exp2} with $\tl_{i,t}=\ell_{i,t}\frac{V_{i,t}}{w_{i,t}}$, 
we have
	$$
	\oR_n \le \frac{d\log 2}{\eta} 
		+ \frac{\eta nd}2.
	$$
In particular for $\eta=\sqrt{\frac{2\log 2}{n}}$, we have
	$
	\oR_n \le d\sqrt{2n \log 2}.
	$
\end{theorem}

Note that as for {\sc linexp}, we end up upper bounding 
$\sum_{i=1}^d \ell_{i,t}$ by $d$. In the case of almost symmetric set $\cS$ of order $k$, this sum can be bounded by $2d/k$, while 
$\log(|\cS|)$ is upper bounded 
by $k\log(d+1)$. So as for {\sc linexp}, this leads to a regret bound of order 
$\sqrt{nd\log d}$ when the set of actions is an almost symmetric set.

\section{Bandit Game} \label{sec:bandit}

The upper bounds for {\sc exp2} in the bandit case proposed in Table \ref{table:2} (page \pageref{table:2}) are extracted from \cite{DHK08}.
The approach proposed by the authors is to use {\sc exp2} in the space described by a barycentric spanner. More precisely,
let $m=\text{dim}(\text{Span}(\cS))$ and $e_1, \hdots, e_m$ be a barycentric spanner of $\cS$; for instance, take $(e_1, \hdots, e_m) \in \argmax_{(x_1, \hdots, x_m) \in \cS^m} |\text{det}_{\text{Span}(\cS)}(x_1,\hdots,x_m)|$ \cite[see][]{AK04}.
We introduce the transformations $T_1 : \R^d \rightarrow \R^m$ such that for $x \in \R^d$,
$T_1(x)= (x^T e_1, \hdots, x^T e_m)^T$, and $T_2 : \cS \rightarrow [-1,1]^m$ such that for $v \in \cS$, $v = \sum_{i=1}^m (T_2(v))_i e_i$. Note that for any $v \in \cS$, we have
$\ell_t^T v = T_1(\ell_t)^T T_2(v)$. Then the loss estimate for $v \in \cS$ is
$$\tilde{\ell}_t^T v = \left( Q_t^+ T_2(V_t) T_2(V_t)^T T_1(\ell_t) \right)^T T_2(v), \; \text{where} \; Q_t = \E_{V \sim p_t} T_2(V) T_2(V)^T.$$
Moreover the authors also add a forced exploration which is uniform over the barycentric spanner.

A concurrent approach is the one proposed in \cite{CL09, CL10}. There the authors study {\sc exp2} directly in the original space, with the estimate described
in Figure \ref{fig:Exp2}, and with an additional forced exploration which is uniform over $\cS$. They work out several examples of sets $\cS$ for which they improve
the regret bound by a factor $\sqrt{d}$ with respect to \cite{DHK08}. Unfortunately there exists sets $\cS$ for which this approach fails to provide a bound polynomial
in $d$. In general one needs to replace the uniform exploration over $\cS$ by an exploration that is tailored to this set. How to do this in general is still an open question.
\newline

The upper bounds for {\sc linexp} in the bandit case proposed in Table \ref{table:2} (page \pageref{table:2}) are derived by using the trick of \cite{DHK08} (that is, by
working with a barycentric spanner). The proof of this result is omitted, since it does not yield the optimal dependency in $n$. Moreover we can not analyze {\sc linpoly} 
since \eqref{eq:wp} is not well defined in this case, because $\tilde{\ell}_t$ can be non-positive. In general we believe that the {\sc lininf} approach is not
sound for the bandit case, and that one needs to work with a Legendre function with non-diagonal Hessian. 
\newline

The only known {\sc cleb} with non-diagonal Hessian is the one proposed in \cite{AHR08}, where the authors use a self-concordant barrier function. In this
case, they are able to propose a loss estimate related to the structure of the Hessian. 
This approach is powerful, and under the $L_2$ assumption leads to a regret upper bound of order $d\sqrt{\theta n \log n}$ for $\theta>0$ such that
$\Conv(\cS)$ admits a $\theta$-self-concordant barrier function \cite[see][section 5]{AHR08}. When $\Conv(\cS)$ admits a $O(1)$-self-concordant barrier function,
the upper bound matches the lower bound $O\big(d\sqrt{n}\big)$. The open question is to determine for which sets $\cS$, this occurs.

\section{Lower Bounds} \label{sec:lb}
We start this Section with a result that shows that {\sc exp2} is suboptimal
against $L_{\infty}$ adversaries. This answers a question of
\cite{KWK10}.
\begin{theorem} \label{lb:expinfty}
Let $n \geq d$. There exists a subset $\cS \subset \{0,1\}^d$ such that in the full information game, for the {\sc exp2} strategy (for any learning rate $\eta$), we have
$$\sup \oR_n \geq 0.02 \, d^{3/2} \sqrt{n},$$
where the supremum is taken over all $L_{\infty}$ adversaries.
\end{theorem}

\begin{proof}
For sake of simplicity we assume here that $d$ is a multiple of $4$ and that $n$ is even. We consider the following subset of the hypercube:
\begin{align*}
& \cS = \bigg\{v \in \{0,1\}^d : \sum_{i=1}^{d/2} v_i = d/4 \;\; \text{and} \;\; \\
& \bigg( v_{i} = 1, \forall i \in \{d/2+1; \hdots, d/2+d/4\}\bigg) \;\; \text{or} \;\; \bigg(v_i = 1, \forall i \in \{d/2+d/4+1, \hdots, d\}\bigg) \bigg\}.
\end{align*}
That is, choosing a point in $\cS$ corresponds to choosing a subset of $d/4$ elements in the first half of the coordinates, and choosing one of the two first disjoint
intervals of size $d/4$ in the second half of the coordinates.
\newline

We will prove that for any parameter $\eta$, there exists an adversary such that Exp (with 
parameter $\eta$) has a regret of at least $\frac{n d}{16} \tanh\big(\frac{\eta d}8\big)$, and that there exists another adversary such that its regret is at least
$\min\big(\frac{d \log 2}{12 \eta}, \frac{n d}{12}\big)$. 
As a consequence, we have
  \begin{align*}
  \sup \oR_n & \ge \max\bigg(\frac{n d}{16} \tanh\Big(\frac{\eta d}8\Big),\min\left(\frac{d \log 2}{12 \eta}, \frac{n d}{12}\right)\bigg) \\
  & \ge \min\bigg(\max\left(\frac{n d}{16} \tanh\Big(\frac{\eta d}8\Big),\frac{d \log 2}{12 \eta}\right), \frac{n d}{12}\bigg)
   \ge \min\Big( A , \frac{n d}{12}\Big),
  \end{align*}
with \begin{align*}
A&=\min_{\eta\in[0,+\infty)}\max\left(\frac{n d}{16} \tanh\Big(\frac{\eta d}8\Big),\frac{d \log 2}{12 \eta}\right)\\
& \ge \min\bigg\{ \min_{\eta d \ge 8} \frac{n d}{16} \tanh\Big(\frac{\eta d}8\Big),
  \min_{\eta d < 8} \max\left(\frac{n d}{16} \tanh\Big(\frac{\eta d}8\Big),\frac{d \log 2}{12 \eta}\right)\bigg\}\\
  & \ge \min\bigg\{ \frac{n d}{16} \tanh(1),
  \min_{\eta d < 8} \max\left(\frac{n d}{16} \frac{\eta d}8 \frac1{\tanh(1)},\frac{d \log 2}{12 \eta}\right)\bigg\}\\
  & \ge \min\Bigg\{ \frac{n d}{16} \tanh(1),
   \sqrt{\frac{n d^3 \log 2}{128\times12\times\tanh(1)}}\Bigg\}\ge \min\big( 0.04\, n d, 0.02 \, d^{3/2} \sqrt{n}\big).
\end{align*}

Let us first prove the lower bound $\frac{n d}{16} \tanh\big(\frac{\eta d}8\big)$. We define the following adversary:
$$\ell_{i,t} = \left\{
\begin{array}{ccc}
1 & \text{if} & i \in \{d/2+1; \hdots, d/2+d/4\} \;\; \text{and} \;\; t \;\; \text{odd},\\
1 & \text{if} & i \in \{d/2+d/4+1, \hdots, d\} \;\; \text{and} \;\; t \;\; \text{even},\\
0 & \text{otherwise}. &
\end{array}
\right.$$
This adversary always put a zero loss on the first half of the coordinates, and alternates between a loss of $d/4$ for choosing the first interval (in the
second half of the coordinates) and the second interval. At the beginning of odd rounds, any vertex $v \in \cS$ has the same cumulative loss and thus Exp picks its expert uniformly at random, 
which yields an expected cumulative loss equal to $n d / 16$. On the other hand at even rounds the probability distribution to select the vertex $v \in \cS$ is always the same. 
More precisely
the probability of selecting a vertex which contains the interval $\{d/2+d/4+1,\hdots,d\}$ (i.e, the interval with a $d/4$ loss at this round) is exactly $\frac{1}{1+\exp(-\eta d /4)}$.
This adds an expected cumulative loss equal to $\frac{n d}{8} \frac{1}{1+\exp(-\eta d /4)}$. Finally note that the loss of any fixed vertex is $n d / 8$. Thus we obtain
\begin{align*}
\oR_n = \frac{n d}{16} + \frac{n d}{8} \frac{1}{1+\exp(-\eta d /4)} - \frac{n d}{8}  = \frac{n d}{16} \tanh\Big(\frac{\eta d}8\Big).
\end{align*}

We move now to the dependency in $1/\eta$. Here we consider the adversary
defined by:
$$\ell_{i,t} = \left\{
\begin{array}{ccc}
1-\epsilon & \text{if} & i \leq d/4, \\
1 & \text{if} & i \in \{d/4+1, \hdots, d/2\}, \\
0 & \text{otherwise}. &
\end{array}
\right.$$

Note that against this adversary the choice of the interval (in the
second half of the components) does not matter.  Moreover by symmetry
the weight of any coordinate in $\{d/4+1,\hdots,d/2\}$ is the same (at
any round). Finally remark that this weight is decreasing with
$t$. Thus we have the following identities (in the big sums 
$i$ represents the number of components selected in the
first $d/4$ components):
\begin{align*}
\oR_n & = \E\bigg(\eps \sum_{t=1}^n \sum_{i=d/4+1}^{d/2} V_{i,t}\bigg)
= \eps \frac{d}4\sum_{t=1}^n \E V_{d/2,t} 
\geq \frac{n \epsilon d}{4} \P(V_{d/2,n}=1) \\
& = \frac{n \epsilon d}{4} \frac{\sum_{v \in \cS : v_{d/2}=1} \exp(- \eta n \ell_2^T v)}{\sum_{v \in \cS} \exp(- \eta n \ell_2^T v)} \\
& = \frac{n \epsilon d}{4}  \frac{\sum_{i=0}^{d/4-1} \binom{d/4}{i} \binom{d/4-1}{d/4-i-1} \exp(- \eta (n d / 4 - i n \epsilon))}
{\sum_{i=0}^{d/4} \binom{d/4}{i} \binom{d/4}{d/4-i} \exp(- \eta (n d / 4 - i n \epsilon))} \\
& = \frac{n \epsilon d}{4}  \frac{\sum_{i=0}^{d/4-1} \binom{d/4}{i} \binom{d/4-1}{d/4-i-1} \exp(\eta i n \epsilon)}
{\sum_{i=0}^{d/4} \binom{d/4}{i} \binom{d/4}{d/4-i}\exp(\eta i n \epsilon)} \\
& = \frac{n \epsilon d}{4} \frac{\sum_{i=0}^{d/4-1} \big(1 - \frac{4i}d\big) \binom{d/4}{i} \binom{d/4}{d/4-i} \exp(\eta i n \epsilon)}
{\sum_{i=0}^{d/4} \binom{d/4}{i} \binom{d/4}{d/4-i} \exp(\eta i n \epsilon)}
\end{align*}
where we used 
$\binom{d/4-1}{d/4-i-1} = \big(1 - \frac{4i}d\big) \binom{d/4}{d/4-i}$ in the last equality. Thus taking 
$\epsilon = \min\big(\frac{\log 2}{\eta n}, 1\big) $ yields
$$\oR_n  \geq \min\left(\frac{d \log 2}{4 \eta}, \frac{n d}{4}\right) \frac{\sum_{i=0}^{d/4-1} \big(1-  \frac{4i}d\big) \binom{d/4}{i}^2 
  \min(2, \exp(\eta n))^i}
  {\sum_{i=0}^{d/4} \binom{d/4}{i}^2 \min(2, \exp(\eta n))^i}
  \geq \min\left(\frac{d \log 2}{12 \eta}, \frac{n d}{12}\right),$$
where the last inequality follows from Lemma \ref{lem:tech1} (see Appendix \ref{sec:AppTech}).
This concludes the proof of the lower bound.
\end{proof}

The next two theorems give lower bounds under the three feedback
assumptions and the two types of adversaries. The cases ($L_2$, Full
Information) and ($L_2$, Bandit) already appeared in \cite{DHK08},
while the case ($L_{\infty}$, Full Information) was treated in
\cite{KWK10} (with more precise lower bounds for subsets $\cS$ of
particular interest). Note that the lower bounds for the semi-bandit
case trivially follow from the ones for the full information
game. Thus our main contribution here is the lower bound for
($L_{\infty}$, Bandit), which is technically quite different from the
other cases.
We also give explicit constants in all cases.

\begin{theorem} \label{th:lbLinfty}
Let $n \geq d$. Against $L_{\infty}$ adversaries 
in the cases of  full information and semi-bandit games, we have
$$R_n \geq 0.008 \; d \sqrt{n} ,$$
and in the bandit game
$$R_n \geq 0.01 \; d^{3/2} \sqrt{n}.$$
\end{theorem}

\begin{proof}
 In this proof we consider the following subset of $\{0,1\}^d$:
$$\cS = \{v \in \{0,1\}^d : \forall i \in \{1,\hdots,\lfloor d/2
\rfloor\}, v_{2i-1} +  v_{2 i} =1 \}.$$ 
Under full
information, playing in $\cS$ corresponds to playing $\lfloor d/2
\rfloor$ independent standard full information games with $2$
experts. Thus we can apply [Theorem 30, \cite{AB10}] to obtain:
$$R_n \geq \lfloor d/2 \rfloor \times 0.03 \sqrt{n \log 2} \geq 0.008
\; d \sqrt{n} .$$ 
We now move to the bandit game, for which the proof
is more challenging. For the sake of simplicity, we assume in the
following that $d$ is even. Moreover, we restrict our attention to
deterministic forecasters, the extension to general forecaster can be
done by a routine application of Fubini's theorem.

\medskip
\noindent {\em First step: definitions.}
\newline

We denote by $I_{i,t} \in \{1,2\}$ the random variable such that $V_{2 i, t} =1$ if and only if $I_{i,t}=2$. That is, $I_{i,t}$ is the expert chosen at time $t$ in the $i^{th}$ game. 
We also define the empirical distribution of plays $q^i_n=(q^i_{1,n},q^i_{2,n})$ in game $i$ as
$q^i_{j,n} = \frac{\sum_{t=1}^n \IND_{I_{i,t}=j}}{n}$.
Let $J_{i,n}$ be drawn according to $q_n^i$.

In this proof we consider a set of $2^{d/2}$ adversaries. For $\alpha=(\alpha_1,\hdots,\alpha_{d/2}) \in \{1,2\}^{d/2}$ we define the $\alpha$-adversary as follows:
For any $t \in \{1,\hdots,n\}$, the loss of expert $\alpha_i$ in game $i$ is drawn from a Bernoulli of parameter $1/2$ while the loss of the other expert in game $i$ is drawn from a Bernoulli of parameter $1/2+\epsilon$. We note $\E_{\alpha}$ when we integrate with respect to the reward
generation process of the $\alpha$-adversary. We note $\P_{i, \alpha}$ the law of $J_{i,n}$ when the forecaster plays against the 
$\alpha$-adversary. Remark that we have $\P_{i, \alpha}(J_{i,n} = j) = \E_{\alpha} \frac{1}{n} \sum_{t=1}^n \IND_{I_{i,t}=j}$, 
hence, against the $\alpha$-adversary we have:
$$\oR_n = \E_{\alpha} \sum_{t=1}^n \sum_{i=1}^{d/2} \epsilon \IND_{I_{i,t} \neq \alpha_i} = n \epsilon \sum_{i=1}^{d/2} \left(1 - \P_{i, \alpha}(J_{i,t}=\alpha_i)\right),$$
which implies (since the maximum is larger than the mean)
\begin{equation} \label{eq:firststepmm}
\sup_{\alpha \in \{1,2\}^{d/2}} \oR_n \geq n \epsilon \sum_{i=1}^{d/2} \left(1 - \frac{1}{2^{d/2}} \sum_{\alpha \in \{1,2\}^{d/2}} \P_{i, \alpha}(J_{i,n} = \alpha_i)\right).
\end{equation}

\noindent {\em Second step: information inequality.}
\newline

Let $\P_{-i,\alpha}$ be the law of $J_{i,n}$ against the adversary which plays like the $\alpha$-adversary except that in the $i^{th}$ game, the losses of both coordinates are drawn from a Bernoulli of parameter $1/2+\epsilon$ (we call it the $(-i,\alpha)$-adversary). Now we use Pinsker's inequality which gives:
$$\P_{i, \alpha}(J_{i,n} = \alpha_i) \leq \P_{- i, \alpha}(J_{i,n} = \alpha_i)  + \sqrt{\frac{1}{2} \K(\P_{- i, \alpha},\P_{i, \alpha})},$$
and thus, (thanks to the concavity of the square root)
\begin{equation} \label{eq:pinsk}
\frac{1}{2^{d/2}} \sum_{\alpha \in \{1,2\}^{d/2}} \P_{i, \alpha}(J_{i,n} = \alpha_i) \leq \frac{1}{2} + \sqrt{\frac{1}{2^{d/2+1}} \sum_{\alpha \in \{1,2\}^{d/2}} \K(\P_{- i, \alpha},\P_{i, \alpha})}.
\end{equation}
\newline

\noindent {\em Third step: computation of $\K(\P_{- i, \alpha},\P_{i, \alpha})$
with the chain rule for Kullback-Leibler divergence.} 
\newline

Note that since the forecaster is deterministic, the sequence of
observed losses (up to time $n$) $W_n \in \{0,1, \hdots, d\}^n$
uniquely determines the empirical distribution of plays $q_n^i$, and
in particular the law of $J_{i,n}$ conditionally to $W_n$ is the same
for any adversary. Thus, if we note $\P_{\alpha}^n$ (respectively
$\P_{-i,\alpha}^n$) the law of $W_n$ when the forecaster plays against
the $\alpha$-adversary (respectively the $(-i,\alpha)$-adversary),
then one can easily prove that
$
\K(\P_{- i, \alpha},\P_{i, \alpha}) \leq \K(\P_{-i,\alpha}^n, \P_{\alpha}^n)
$. 
Now we use the chain rule for Kullback-Leibler divergence
iteratively to introduce the laws $\P^t_{\alpha}$ of the observed
losses $W_t$ up to time $t$. More precisely, we have,
\begin{align*}
\K(\P_{-i,\alpha}^n, \P_{\alpha}^n) 
& =  \K(\P_{-i,\alpha}^1, \P_{\alpha}^1) + \sum_{t=2}^n \sum_{w_{t-1} \in \{0,1,\hdots,d\}^{t-1}} \P_{-i,\alpha}^{t-1}(w_{t-1}) \K(\P_{-i,\alpha}^t(. | w_{t-1}),\P_{\alpha}^t(. | w_{t-1})) \\
& =  \K\left(\cB_{\emptyset}, \cB_{\emptyset}'\right) \IND_{I_{i,1} = \alpha_i} + \sum_{t=2}^n \sum_{w_{t-1} : I_{i,t} = \alpha_i} \P_{-i,\alpha}^{t-1}(w_{t-1}) \K\left(\cB_{w_{t-1}}, \cB_{w_{t-1}}'\right),
\end{align*}
where $\cB_{w_{t-1}}$ and $\cB_{w_{t-1}}'$ are sums of $d/2$ Bernoulli distributions with parameters in $\{1/2,1/2+\epsilon\}$ and such that the number of Bernoullis with parameter $1/2+\epsilon$ in $\cB_{w_{t-1}}$ is equal to the number of Bernoullis with parameter $1/2+\epsilon$ in $\cB_{w_{t-1}}'$ plus one. 
Now using Lemma \ref{lem:KLbinomials} (see Appendix \ref{sec:AppTech}) we obtain
$\K(\P_{-i,\alpha}^n, \P_{\alpha}^n) \leq \frac{16 \; \epsilon^2}{d} \E_{-i,\alpha} \sum_{t=1}^n \IND_{I_{i,t}=\alpha_i}$.
Summing and plugging this into \eqref{eq:pinsk} we obtain 
$\frac{1}{2^{d/2}} \sum_{\alpha \in \{1,2\}^{d/2}} \P_{i, \alpha}(J_{i,n} = \alpha_i) \leq \frac{1}{2} + 2 \epsilon \sqrt{\frac{n}{d}}$.
To conclude the proof one needs to plug in 
this last equation in \eqref{eq:firststepmm} along with straightforward computations.
\end{proof}

\begin{theorem} \label{th:lbL2}
Let $n \geq d$. Against $L_{2}$ adversaries 
in the cases of  full information and semi-bandit games, we have
$$R_n \geq 0.05 \sqrt{d n} ,$$
and in the bandit game
$$R_n \geq 0.05 \min(n , d \sqrt{n}).$$
\end{theorem}

\bibliography{biblio}

\newpage

\appendix

\section{Standard prediction games}
\label{sec:inf}

\begin{figure}[t]
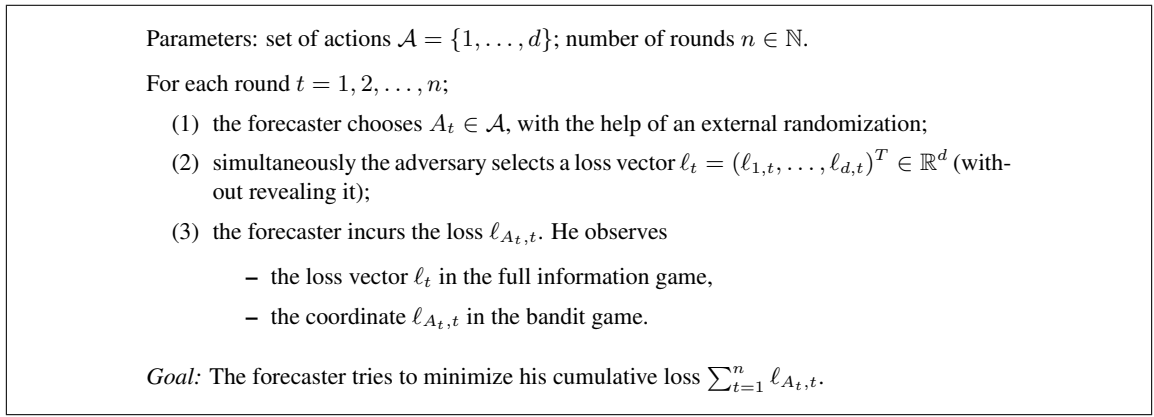

\bookbox{\small
{Parameters:} set of actions $\cA = \{1, \dots ,d\}$; number of rounds $n \in \N$.

\medskip\noindent
For each round $t=1,2,\ldots,n$;
\begin{itemize}
\item[(1)]
the forecaster chooses $A_t \in \cA$, with the help of an external randomization;
\item[(2)]
simultaneously the adversary selects a loss vector $\ell_t=(\ell_{1,t},\dots,\ell_{d,t})^T\in\R^d$ (without revealing it);
\item[(3)]
the forecaster incurs the loss $\ell_{A_t,t}$. He observes
\begin{itemize}
\item the loss vector $\ell_t$ in the full information game,
\item the coordinate $\ell_{A_t,t}$ in the bandit game.
\end{itemize}
\end{itemize}

\medskip\noindent
{\em Goal:} The forecaster tries to minimize his cumulative loss $\sum_{t=1}^n \ell_{A_t,t}$.
}
\caption{Standard prediction games.}
\label{fig:notcomband}
\end{figure}

It is well-known that the standard prediction games described in Figure \ref{fig:notcomband}
are specific cases of the combinatorial prediction games described in Figure \ref{fig:comband}.
Indeed, consider $\cS=\{\ba_1,\dots,\ba_d\}$, where 
$\ba_i\in\{0,1\}^d$ is the vector whose only nonzero component is the $i$-th one. 
The standard and combinatorial prediction games are then equivalent by using $V_t=\ba_{A_t}$ and noticing that 
$\ell_t^T\ba_i = \ell_{i,t}$.
In particular, the semi-bandit and bandit combinatorial prediction games are then both
equivalent to the traditional multi-armed bandit game.

We now show that {\sc inf} (defined in Figure 2 of \cite{AB10}) is a special case of {\sc lininf}.

\begin{proof}
Indeed, suppose that the estimates $\tl_1,\dots,\tl_n$ are nonnegative (coordinate-wise),
and take $w_1=\big(\frac1d,\dots,\frac1d\big)$.
Then the vector $w'_{t+1}$ satisfying \eqref{eq:wp} exists, and is defined coordinate-wise by
  $
   \psi^{-1}(w'_{i,t+1}) = \psi^{-1}(w_{i,t}) - \tl_{i,t}.
  $
Besides, the optimality condition of \eqref{eq:proj} implies the existence of $c_t\in\R$ (independent of $i$)
such that
  $
  \psi^{-1}(w_{i,t+1}) = \psi^{-1}(w'_{i,t+1}) + c_t.
  $
It implies $\psi^{-1}(w_{i,t})= \psi^{-1}(w_{i,1})-\sum_{s=1}^{t-1} (\tl_{i,s}-c_s)$ for any $t\ge1$.
So there exists $C_t\in\R$ such that $w_{i,t}=\psi\big(\sum_{s=1}^{t-1} (1-\tl_{i,s})-C_t\big)$.
Since $w_t\in\Conv(\cS)$, the constant $C_t$ should satisfy $\sum_{i=1}^n w_{i,t} =1$.
We thus recover {\sc inf} with the estimate $1-\tl_{i,s}$ of the reward $1-\ell_{i,t}$. So the Bregman projection has here a simple solution depending on a unique constant $C_t$ obtained by solving 
the equality $\sum_{i=1}^n \psi\big(\sum_{s=1}^{t-1} (1-\tl_{i,s})-C_t\big) = 1$.
\end{proof}

Next we show how to obtain the minimax $\sqrt{n d}$ regret bound, with a much simpler proof than the one proposed in \cite{AB10}, as well as a better constant.

\begin{theorem} \label{th:polyinf}
Let $q>1$.
For the {\sc inf} forecaster \big(that is for \cleb\  with $w_1=\big(\frac1d,\dots,\frac1d\big)^T$ and $\cS=\{\ba_1,\dots,\ba_d\}$\big)
using $\psi(x)=(-\eta x)^{-q}$ and $\tl_{i,t}=\ell_{i,t} \frac{V_{i,t}}{w_{i,t}}$, we have
	$$
	\oR_n \le \frac{q d^{\frac1q}}{\eta(q-1)} 
		+ \frac{q\eta n d^{1-\frac1q}}2.
	$$
In particular for $\eta=\sqrt2d^{\frac1q-\frac12}[(q-1)n]^{-\frac12}$, we have
$\oR_n \le q \sqrt{\frac{2nd}{q-1}}$. 
\end{theorem}

In view of this last bound, the optimal $q$ is $q=2$, which leads to 
$\oR_n \le 2 \sqrt{2nd}$. This improves on the bound 
$\oR_n \le 8 \sqrt{nd}$ obtained in Theorem 11 of \cite{AB10}.
The {\sc inf} forecaster with the above polynomial $\psi$ is referred to as {\sc polyinf}
in Figure \ref{fig:organi}.

\begin{proof}
We apply \eqref{eq:psic}. First we bound the divergence term.
We have $\psi^{-1}(s) = -\frac1\eta s^{-1/q}$, and 
$	
D_F(u,v) = \frac1\eta \sum_{i=1}^d \Big(\frac1{q-1}v_i^{1-\frac1q} 
	- \frac{q}{q-1}u_i^{1-\frac1q} + u_i v_i^{-\frac1q} \Big),
$
hence
	\begin{align*}
	\max_{u\in\Conv(\cS)} D_F(u,w_1)
		= D_F\Big(\big(1,0,\dots,0\big)^T,\big(\frac1d,\dots,\frac1d\big)^T\Big)
		= \frac{q}{\eta(q-1)} \big(d^{\frac1q}-1\big).
	\end{align*}
Combining this with $(\psi^{-1})'(w_{i,t})=\frac1{q\eta} w_{i,t}^{-1-\frac1q}$ 
and \eqref{eq:psic}, we obtain
	\begin{align*}
	\oR_n & \le \frac{q d^{\frac1q}}{\eta(q-1)} 
		+ \frac{q\eta}2 \E \sum_{t=1}^n\sum_{i=1}^d w_{i,t}^{1+\frac1q} \tl_{i,t}^2\\
	& = \frac{q d^{\frac1q}}{\eta(q-1)} 
		+ \frac{q\eta}2 \sum_{t=1}^n\sum_{i=1}^d \E\big( w_{i,t}^{\frac1q-1} V_{i,t} \ell_{i,t}^2 \big)
    \le \frac{q d^{\frac1q}}{\eta(q-1)} 
		+ \frac{q\eta}2 \sum_{t=1}^n\sum_{i=1}^d \E\big( w_{i,t}^{\frac1q} \big)
    \le \frac{q d^{\frac1q}}{\eta(q-1)} 
		+ \frac{q\eta}2 n d^{1-\frac1q},
	\end{align*}
where in the last step we use that by H\"older's inequality,
$\sum_{i=1}^d ( w_{i,t}^{\frac1q} \times 1 )\le \big( \sum_{i=1}^d w_{i,t} \big)^{\frac1q} \times d^{1-\frac1q}$.
\end{proof}


%
%

\section{Proofs of Theorems in Section \ref{sec:fullinfo}} \label{fi:app}
\subsection*{Proof of Theorem \ref{fi:linexp:infty}}
We have 
	$D_F(u,v)=\frac1\eta \sum_{i=1}^d 
	\Big( u_i\log\big(\frac{u_i}{v_i}\big)-u_i+v_i\Big),$
hence from the Pythagorean theorem,
	$$
	D_F(u,w_1)\le D_F\big(u,(1,\dots,1)^T\big)\le\frac{d}\eta.
	$$
Since we have $\sum_{i=1}^d w_{i,t}\le d$, 
Theorem \ref{th:lininf} implies
	$
	\oR_n \le \frac{d}\eta + \frac{\eta}2 \E \sum_{t=1}^n \sum_{i=1}^d 
		w_{i,t} \ell_{i,t}^2 \le \frac{d}\eta + \frac{nd\eta}2.
	$

\subsection*{Proof of Theorem \ref{fi:linexp:2}}
As in the previous proof, we have
	$
	D_F(u,w_1)\le\frac{d}\eta,
	$
but under the $L_2$ constraint, we can improve the bound on $\sum_{i=1}^d w_{i,t} \ell_{i,t}^2$
by using $\sum_{i=1}^d w_{i,t} \ell_{i,t} \le 1$ (since $w_t\in\Conv(\cS)$).
This gives $\oR_n \le \frac{d}\eta + \frac{n\eta}2.$

\subsection*{Proof of Theorem \ref{fi:linpoly:infty}}
We have 
	$D_F(u,v)=\frac1\eta \sum_{i=1}^d 
	\Big( \frac1{q-1} v_i^{1-\frac1q} - \frac{q}{q-1} u_i^{1-\frac1q} 
		+ u_i v_i^{-\frac1q} \Big),$
hence 
	$
	D_F(u,w_1)\le\frac{d}{\eta(q-1)}.
	$
Since we have $w_{i,t}^{1+\frac1q} \ell_{i,t}^2\le 1$, 
Theorem \ref{th:lininf} implies
	$
	\oR_n \le \frac{d}{\eta(q-1)} + \frac{\eta q}2 \E \sum_{t=1}^n \sum_{i=1}^d 
		w_{i,t}^{1+\frac1q} \ell_{i,t}^2 \le \frac{d}{\eta(q-1)} + \frac{ndq\eta}2.
	$

\subsection*{Proof of Theorem \ref{fi:linpoly:2}}
As in the previous proof, we have
	$
	D_F(u,w_1)\le\frac{d}{\eta(q-1)}.
	$
Under the $L_2$ constraint, we can improve the bound on $\sum_{i=1}^d w_{i,t}^{1+\frac1q} \ell_{i,t}^2$
by using $\sum_{i=1}^d w_{i,t} \ell_{i,t} \le 1$ (since $w_t\in\Conv(\cS)$).
This gives 
	$$
	\oR_n \le \frac{d}{\eta(q-1)} + \frac{\eta q}2 \E \sum_{t=1}^n \sum_{i=1}^d 
		w_{i,t}^{1+\frac1q} \ell_{i,t}^2 \le \frac{d}{\eta(q-1)} + \frac{nq\eta}2.
	$$

\subsection*{Proof of Theorem \ref{fi:exp2:infty}}
Using $\log(|\cS|)\le d\log 2$, $0\le \tl_t^Tv\le d$ and $\sum_{v \in \cS} w_{v,t}=1$ in
Theorem \ref{th:Exp2}, we get the result.

\subsection*{Proof of Theorem \ref{fi:exp2:2}}
Using $\log(|\cS|)\le d\log 2$, $0\le \tl_t^Tv\le 1$ and $\sum_{v \in \cS} w_{v,t}=1$ in
Theorem \ref{th:Exp2}, we get the result.

\section{Proofs of Theorems in Section \ref{sec:semibandit}} \label{sb:app}
\subsection*{Proof of Theorem \ref{sb:linexp:infty}}
We have again
	$
	D_F(u,w_1)\le\frac{d}\eta.
	$
Since we have $\sum_{i=1}^d \ell_{i,t}\le d$, 
Theorem \ref{th:lininf} implies
	$$
	\oR_n \le \frac{d}\eta + \frac{\eta}2 \E \sum_{t=1}^n \sum_{i=1}^d 
		\ell_{i,t}^2 \frac{V_{i,t}}{w_{i,t}} = \frac{d}\eta + \frac{\eta}2 
		\sum_{t=1}^n \sum_{i=1}^d 
		\E(\ell_{i,t}^2)\le \frac{d}\eta + \frac{nd\eta}2.
	$$

\subsection*{Proof of Theorem \ref{sb:linexp:infty2}}
The starting point is Theorem \ref{th:lininf}, which, by using 
    $\E \big(\ell_{i,t}^2 \frac{V_{i,t}}{w_{i,t}}\big) = \E \ell_{i,t}^2 \le \E \ell_{i,t} $, implies
	\begin{equation} \label{eq:asa}
	\oR_n \le D_F\big(u,w_1) + \frac{\eta}2 \E \sum_{t=1}^n \sum_{i=1}^d 
		\ell_{i,t}^2 \frac{V_{i,t}}{w_{i,t}} \le D_F\big(u,w_1) + \frac{\eta}2 
		\E \sum_{t=1}^n \sum_{i=1}^d \ell_{i,t}.
	\end{equation}
For any $u\in[0,1]^d$ such that $\sum_{i=1}^d u_i\le k$, we have 
    \begin{equation} \label{eq:asb}
    D_F\big(u,w_1)\le D_F\bigg(u,\Big(\frac{k}d,\dots,\frac{k}d\Big)^T\bigg)
\le \frac1\eta\bigg(k+\sum_{i=1}^d u_i \log\Big(\frac{du_i}{ke}\Big)\bigg)
\le 
\frac{k\cL}\eta,
    \end{equation}
where the last inequality can be obtained by writing the optimality conditions. More precisely, two cases are considered depending on whether holds $\sum_{i=1}^d u_i=k$ at the optimum: when it is the case, the maximum is achieved for $u$ of the form $u=(1,\dots,1,0,\dots,0)^T$; otherwise, $u=(0,\dots,0)^T$ achieves the maximum. 
The desired results are then obtained by combining \eqref{eq:asa}, \eqref{eq:asb} and an upper bound on $\sum_{i=1}^d \ell_{i,t}$: indeed, under the $L_\infty$ assumption, we have $\sum_{i=1}^d \ell_{i,t}\le d$.
Under the $L_2$ assumption, since $\cS$ is an almost symmetric set of order $k$, there exists $z\in\Conv(\cS) \cap \big[\frac{k}{2d};1\big]^d$, and consequently
	$\sum_{i=1}^d \ell_{i,t} \le \sum_{i=1}^d \big(\frac{2d}k z_{i}\big) \ell_{i,t} \le \frac{2d}k.$

\subsection*{Proof of Theorem \ref{sb:linpoly:infty}}
We have again
	$
	D_F(u,w_1)\le\frac{d}{\eta(q-1)}.
	$
Since we have $w_{i,t}^{\frac1q} \ell_{i,t}^2\le 1$, 
Theorem \ref{th:lininf} implies
	$$
	\oR_n 
	\le \frac{d}{\eta(q-1)} + \frac{\eta q}2 \sum_{t=1}^n \sum_{i=1}^d 
		\E\big( w_{i,t}^{\frac1q} \ell_{i,t}^2 \big)
		\le \frac{d}{\eta(q-1)} + \frac{ndq\eta}2.
	$$

\subsection*{Proof of Theorem \ref{sb:linpoly:2}}
We have again
	$
	D_F(u,w_1)\le\frac{d}{\eta(q-1)}.
	$
From 
	$\E\big(w_{i,t}^{1+\frac1q} \tl_{i,t}^2\big)
		= \E\big(w_{i,t}^{\frac1q} \ell_{i,t}^2\big) \le \E\big[(w_{i,t} \ell_{i,t})^{\frac1q}\big]$ and Theorem \ref{th:lininf}, we get
	$$
	\oR_n 
	\le \frac{d}{\eta(q-1)} + \frac{\eta q}2 \sum_{t=1}^n \sum_{i=1}^d 
		\E\big[(w_{i,t} \ell_{i,t})^{\frac1q}\big]
		\le \frac{d}{\eta(q-1)} + \frac{n d^{1-\frac1q}q\eta}2.
	$$
where we use 
$\sum_{i=1}^d ( w_{i,t} \ell_{i,t} )^{\frac1q} \le \big( \sum_{i=1}^d w_{i,t} \ell_{i,t} \big)^{\frac1q} \times d^{1-\frac1q}$
in the last step.

\subsection*{Proof of Theorem \ref{sb:exp2:infty}}
Let $q_{i,t}=\sum_{v\in\cS:v_i=1} p_{v,t} = \E_{V_t\sim p_t}V_{i,t}$ for $i\in\{1\ldots,d\}$.
We have
\begin{align*}
\E_{V_t\sim p_t} \sum_{v \in \cS} p_{v,t} (\tilde{\ell}_t^T v)^2
 & = \E_{V_t\sim p_t, V_t'\sim p_t} (\tilde{\ell}_t^T V_t')^2 \\
& = \E_{V_t\sim p_t, V_t' \sim p_t} \sum_{i, j} \frac{\ell_{i,t} V_{i,t} \ell_{j,t} V_{j,t}}{q_{i,t} q_{j,t}} V_{i,t}' V_{j,t}' \\
& \le \E_{V_t\sim p_t, V_t'\sim p_t} \sum_{i, j} \ell_{i,t} \ell_{j,t}\frac{V_{i,t}}{q_{i,t}} \frac{V_{j,t}'}{q_{j,t}} 
= \bigg(\sum_{i=1}^d \ell_{i,t}\bigg)^2 \le d^2.
\end{align*}
Using $\log(|\cS|)\le d\log 2$, the result then follows from Theorem \ref{th:Exp2}.

\subsection*{Proof of Theorem \ref{sb:exp2:2}}
Let $q_{i,t}=\sum_{v\in\cS:v_i=1} p_{v,t} = \E_{V_t\sim p_t}V_{i,t}$ for $i\in\{1,\ldots,d\}$.
We have
\begin{align*}
\E_{V_t\sim p_t} \sum_{v \in \cS} p_{v,t} (\tilde{\ell}_t^T v)^2
& = \E_{V_t\sim p_t, V_t' \sim p_t} \sum_{i, j} \frac{\ell_{i,t} V_{i,t} \ell_{j,t} V_{j,t}}{q_{i,t} q_{j,t}} V_{i,t}' V_{j,t}' \\
& \leq \E_{V_t, V_t'} \sum_{i, j} \frac{\ell_{i,t} V_{i,t}}{q_{i,t}} \frac{V_{j,t}'}{q_{j,t}} \ell_{j,t} V_{j,t} \\
& = \E_{V_t} \sum_{i=1}^d \frac{\ell_{i,t} V_{i,t}}{q_{i,t}} \sum_{j=1}^d  \ell_{j,t} V_{j,t} 
\le \E_{V_t} \sum_{i=1}^d \frac{\ell_{i,t} V_{i,t}}{q_{i,t}} = \sum_{i=1}^d \ell_{i,t} \le d. 
\end{align*}
Using $\log(|\cS|)\le d\log 2$, the result then follows from Theorem \ref{th:Exp2}.

\section{Proof of Theorem \ref{th:lbL2}}
We consider the bandit game first. We use the notation and
adversaries defined in the proof of Theorem \ref{th:lbLinfty}. We
modify these adversaries as follows: at each turn one selects
uniformly at random $E_t \in \{1, \hdots, d\}$. Then, at time $t$, the
losses of all coordinates but $E_t$ are set to $0$. This new adversary
is clearly in $L_2$. For this new set of adversaries, one has to do
only two modifications in the proof of Theorem
\ref{th:lbLinfty}. First \eqref{eq:firststepmm} is replaced by:
$$\sup_{\alpha \in \{1,2\}^{d/2}} \oR_n \geq \frac{n \epsilon}{d} \sum_{i=1}^{d/2} \left(1 - \frac{1}{2^{d/2}} \sum_{\alpha \in \{1,2\}^{d/2}} \P_{i, \alpha}(J_{i,n} = \alpha_i)\right).$$
Second $\cB_{w_{t-1}}$ is now a Bernoulli with mean $\mu_t \in \left[\frac12 + \frac{\epsilon}d , \frac12 + \frac{\epsilon}2\right]$ and $\cB_{w_{t-1}}'$ is a Bernoulli with mean $\mu_t - \frac{\epsilon}d$, and thus we have
$$\K\left(\cB_{w_{t-1}},  \cB_{w_{t-1}}'\right) \leq \frac{4 \epsilon^2}{(1 - \epsilon^2) d^2} .$$
The proof is then concluded again with straightforward computations.
\newline

The proof for the full information game is exactly the same as the one
for bandit information, except that the definition of $W_t$ is
slightly different and implies that $\cB_{w_{t-1}}$ is now a Bernoulli
with mean $\frac1d \left(\frac12 + \epsilon\right)$ and
$\cB_{w_{t-1}}'$ is a Bernoulli with mean $\frac1{2d}$, which gives
$$\K\left(\cB_{w_{t-1}},  \cB_{w_{t-1}}'\right) \leq \frac{4 \epsilon^2}{2 d - 1}.$$

\section{Technical lemmas} \label{sec:AppTech}
We prove here two technical lemmas that were used in the proofs above.

\begin{lemma} \label{lem:tech1}
For any $k \in \N^*,$ for any $1\le c \le 2$, we have  
$$\frac{\sum_{i=0}^{k} (1- i/k) \binom{k}{i}^2 c^i}
{\sum_{i=0}^{k} \binom{k}{i}^2 c^i} \geq 1/3.$$
\end{lemma}

\begin{proof}
Let $f(c)$ denote the left-hand side term of the inequality. Introduce
the random variable $X$, which is equal to $i\in\{0,\ldots,k\}$ with
probability $\binom{k}{i}^2 c^i\big/ \sum_{j=0}^k \binom{k}{j}^2 c^j$.
We have 
$f'(c)= \frac1c \E[X(1-X/k)] - \frac1c \E(X)\E(1-X/k) = -
\frac1c \Var X \le 0.$ 
So the function $f$ is decreasing on $[1,2]$,
and, from now on, we consider $c=2$.  Numerator and denominator of the
left-hand side (l.h.s.) differ only by the $1-i/k$ factor.  A lower
bound for the left-hand side can thus be obtained by showing that
the terms for $i$ close to $k$ are not essential to the value of the
denominator.  To prove this, we may use the Stirling formula: for any
$n\ge 1$
  \begin{align} \label{eq:stirl}
  \Big(\frac{n}e\Big)^n \sqrt{2\pi n} < n! < \Big(\frac{n}e\Big)^n \sqrt{2\pi n} e^{1/(12n)}
  \end{align}
Indeed, this inequality implies that for any $k\ge 2$ and $i\in[1,k-1]$
  \begin{align*}
  \Big(\frac{k}i\Big)^i \Big(\frac{k}{k-i}\Big)^{k-i} \frac{\sqrt{k}}{\sqrt{2\pi i(k-i)}} e^{-1/6} < \binom{k}{i} 
    < \Big(\frac{k}i\Big)^i \Big(\frac{k}{k-i}\Big)^{k-i} \frac{\sqrt{k}}{\sqrt{2\pi i(k-i)}} e^{1/12},
  \end{align*}
hence
  \begin{align*}
  \Big(\frac{k}i\Big)^{2i} \Big(\frac{k}{k-i}\Big)^{2(k-i)} \frac{k e^{-1/3}}{{2\pi i(k-i)}}  < \binom{k}{i}^2
    < \Big(\frac{k}i\Big)^{2i} \Big(\frac{k}{k-i}\Big)^{2(k-i)} \frac{k e^{1/6}}{{2\pi i}}  
  \end{align*}
Introduce $\lam=i/k$ and $\chi(\lam)=\frac{2^\lam}{\lam^{2\lam}(1-\lam)^{2(1-\lam)}}$.
We have 
  \begin{align} \label{eq:stirlam}
  [\chi(\lam)]^k \frac{2 e^{-1/3}}{\pi k}  
    < \binom{k}{i}^2 2^i
    < [\chi(\lam)]^k \frac{e^{1/6}}{2\pi \lam}.
  \end{align}
Lemma \ref{lem:tech1} can be numerically verified for $k\le 10^6$.
We now consider $k>10^6$. 
For $\lam\ge 0.666$, since the function $\chi$ can be shown to be decreasing on $[0.666,1]$, the inequality
  $\binom{k}{i}^2 2^i < [\chi(0.666)]^k \frac{e^{1/6}}{2\times 0.666 \times  \pi}$ holds.
We have $\chi(0.657)/\chi(0.666)>1.0002$. Consequently, for 
$k>10^6$, we have $[\chi(0.666)]^k < 0.001 \times [\chi(0.657)]^k/k^2$.
So for $\lam\ge 0.666$ and $k>10^6$, we have
  \begin{align} 
  \binom{k}{i}^2 2^i < 0.001 \times [\chi(0.657)]^k \frac{e^{1/6}}{2\pi\times0.666 \times k^2}
  & <[\chi(0.657)]^k  \frac{2 e^{-1/3}}{1000 \pi k^2}  \notag  \\
  & = \min_{\lam\in[0.656,0.657]}[\chi(\lam)]^k \frac{2 e^{-1/3}}{1000 \pi k^2} \notag \\
  & < \frac1{1000k}\max_{i\in\{1,\dots,k-1\}\cap[0,0.666k)} \binom{k}{i}^2 2^i.  \label{eq:cuta}
  \end{align}
where the last inequality comes from \eqref{eq:stirlam} and the fact that there exists $i\in\{1,\dots,k-1\}$
such that $i/k\in[0.656,0.657]$. Inequality \eqref{eq:cuta} implies that 
for any $i\in\{1,\dots,k\}$, we have
  $$
  \sum_{\frac56 k\le i \le k}\binom{k}{i}^2 2^i 
  < \frac1{1000}\max_{i\in\{1,\dots,k-1\}\cap[0,0.666k)} \binom{k}{i}^2 2^i
  < \frac1{1000}\sum_{0\le i < 0.666 k}\binom{k}{i}^2 2^i .
  $$
To conclude, introducing $A=\sum_{0\le i <0.666 k}\binom{k}{i}^2 2^i$, we have
  \begin{align*}
  \frac{\sum_{i=0}^{k} (1- i/k) \binom{k}{i} \binom{k}{k-i} 2^i}
{\sum_{i=0}^{k} \binom{k}{i} \binom{k}{k-i} 2^i} >
  \frac{(1-0.666) A}{A+0.001A} \ge \frac13.
  \end{align*}
\end{proof}

\begin{lemma} \label{lem:KLbinomials}
Let $\ell$ and $n$ be integers with $\frac12\le \frac{n}2\le \ell\le n$.
Let $p,p',q,p_1,\dots,p_n$ be real numbers in $(0,1)$ with $q\in\{p,p'\}$, $p_1=\cdots=p_\ell=q$ and 
$p_{\ell+1}=\cdots=p_n$.
Let $\cB$ (resp. $\cB'$) be the sum of $n+1$ independent Bernoulli distributions with parameters
$p,p_1,\dots,p_n$ (resp. $p',p_1,\dots,p_n$). We have
$$\KL(\cB, \cB') \le \frac{2(p'-p)^2}{(1-p')(n+2)q}.$$
\end{lemma}

\begin{proof}
Let $Z,Z',Z_1,\dots,Z_n$ be independent Bernoulli distributions with parameters $p,p',p_1,\dots,p_n$.
Define $S=\sum_{i=1}^\ell Z_i$, $T=\sum_{i=\ell+1}^n Z_i$ and $V=Z+S$.
By slight abuse of notation, merging in the same notation the distribution and the random variable, we have
  \begin{align*}
  \KL(\cB, \cB') & = \KL\big((Z+S)+T,(Z'+S)+T\big) \\
  & \le \KL\big((Z+S,T),(Z'+S,T)\big) \\
  & = \KL\big(Z+S,Z'+S\big).
  \end{align*}
Let $s_k=\P(S=k)$ for $k=-1,0,\dots,\ell+1$.
Using the equalities 
  \begin{multline*}
  s_k = \binom{\ell}{k} q^k(1-q)^{\ell-k}
  = \frac{q}{1-q} \frac{\ell-k+1}k \binom{\ell}{k-1} q^{k-1}(1-q)^{\ell-k+1}
  = \frac{q}{1-q} \frac{\ell-k+1}k s_{k-1},
  \end{multline*}
which holds for $1 \le k\le \ell+1$, we obtain
  \begin{align}
  \KL(Z+S,Z'+S) & = \sum_{k=0}^{\ell+1} \P(V=k) 
    \log\bigg(\frac{\P(Z+S=k)}{\P(Z'+S=k)}\bigg) \notag\\
  & = \sum_{k=0}^{\ell+1} \P(V=k)
  \log \bigg(\frac{p s_{k-1}+(1-p) s_{k}}{p' s_{k-1}+(1-p') s_{k}}\bigg) \notag\\
  & = \sum_{k=0}^{\ell+1} \P(V=k)
    \log \bigg(\frac{p \frac{1-q}{q} k+(1-p)(\ell-k+1) }
    {p'\frac{1-q}{q} k+(1-p')(\ell-k+1) }\bigg) \notag\\
  & = \E
    \log \bigg(\frac{(p-q)V+(1-p)q(\ell+1) }
    {(p'-q)V+(1-p')q(\ell+1) }\bigg). \label{eq:klgen}
  \end{align}
\noindent {\em First case: $q=p'$.}
\newline
By Jensen's inequality, using that $\E V=p'(\ell+1)+p-p'$ in this case, we then get
  \begin{align*}
  \KL(Z+S,Z'+S) & \le
    \log \bigg(\frac{(p-p')\E( V )+(1-p)p'(\ell+1) }
    {(1-p')p'(\ell+1) }\bigg) \\
  & = \log \bigg(\frac{(p-p')^2+(1-p')p'(\ell+1) }
    {(1-p')p'(\ell+1) }\bigg) \\
  & = \log \bigg(1+\frac{(p-p')^2}
    {(1-p')p'(\ell+1) }\bigg) 
    \le \frac{(p-p')^2}
    {(1-p')p'(\ell+1) }.
  \end{align*}
\noindent {\em Second case: $q=p$.}
\newline
In this case, $V$ is a binomial distribution with parameters $\ell+1$ and $p$. 
From \eqref{eq:klgen}, we~have
  \begin{align}
  \KL(Z+S,Z'+S) & \le - \E
    \log \bigg(\frac{(p'-p)V+(1-p')p(\ell+1) }{(1-p)p(\ell+1) }
    \bigg)  \notag \\
   & \le - \E \log \bigg(1+\frac{(p'-p)(V-\E V)}{(1-p)p(\ell+1) }
    \bigg) . \label{eq:qp}
  \end{align}  
To conclude, we will use the following lemma.
\begin{lemma} \label{lem:log4}
The following inequality holds for any $x\ge x_0$ with $x_0\in(0,1)$:
  $$
  -\log(x) \le -(x-1)+\frac{(x-1)^2}{2x_0}.
  $$
\end{lemma}
\begin{proof}
Introduce $f(x)=-(x-1)+\frac{(x-1)^2}{2x_0}+\log(x)$.
We have $f'(x)=-1 + \frac{x-1}{x_0} +\frac1x$,
and $f''(x)=\frac1{x_0} -\frac1{x^2}$.
From $f'(x_0)=0$, we get that $f'$ is negative on $(x_0,1)$
and positive on $(1,+\infty)$. This leads to
$f$ nonnegative on $[x_0,+\infty)$.
\end{proof}
  
Finally, from Lemma \ref{lem:log4} and \eqref{eq:qp}, using
$x_0=\frac{1-p'}{1-p}$, we obtain
  \begin{align*}
  \KL(Z+S,Z'+S) & \le \bigg(\frac{p'-p}{(1-p)p(\ell+1)}\bigg)^2 
    \frac{\E[(V-\E V)^2]}{2x_0}\\
   & = \bigg(\frac{p'-p}{(1-p)p(\ell+1)}\bigg)^2 
    \frac{(\ell+1)p(1-p)^2}{2(1-p')}\\
   & =  \frac{(p'-p)^2}{2(1-p')(\ell+1)p}.
  \end{align*}  
\end{proof}

\end{document}